\documentclass[12pt]{article}
\usepackage{enumitem}
\usepackage{soul} 
\usepackage[colorlinks]{hyperref}            
\usepackage{color}
\usepackage{graphicx,subfigure,amsmath,amssymb,amsfonts,bm,epsfig,epsf,url,dsfont}
\usepackage{times}
\usepackage{bbm}      
\usepackage{booktabs}
\usepackage{cases}
\usepackage{fullpage}
\usepackage[small,bf]{caption}
\usepackage{algorithm}
\usepackage{algpseudocode}
\usepackage{natbib}
\usepackage[top=1in,bottom=1in,left=0.7in,right=0.7in]{geometry}

\renewcommand{\phi}{\varphi}

\renewcommand{\P}{\mathbb{P}}
\newcommand{\E}{\mathbb{E}}
\newcommand{\N}{\mathbb{N}}

\newcommand{\cB}{\mathcal{B}}

\def\ds1{\mathds{1}}
\renewcommand{\epsilon}{\varepsilon}

\renewcommand{\tilde}{\widetilde}

\newlength{\minipagewidth}
\newcommand{\beq}{\begin{equation}}
\newcommand{\eeq}{\end{equation}}

\newcommand{\beqa}{\begin{eqnarray}}
\newcommand{\eeqa}{\end{eqnarray}}

\newcommand{\beqan}{\begin{eqnarray*}}
\newcommand{\eeqan}{\end{eqnarray*}}

\def\ba#1\ea{\begin{align*}#1\end{align*}} 
\def\banum#1\eanum{\begin{align}#1\end{align}} 

\def \cR {\mathcal{R}}

\newtheorem{theorem}{Theorem}

\newtheorem{lemma}{Lemma}

\newtheorem{remark}{Remark}
\newtheorem{proposition}{Proposition}

\newcommand{\BlackBox}{\rule{1.5ex}{1.5ex}}  
\newenvironment{proof}{\par\noindent{\bf Proof\ }}{\hfill\BlackBox\\[2mm]}

\begin{document}

\title{Non-Stochastic Multi-Player Multi-Armed Bandits: \\
Optimal Rate With Collision Information, Sublinear Without}

\author{S\'ebastien Bubeck \\
Microsoft Research
\and Yuanzhi Li \thanks{Part of this work was done while Y. Li, Y. Peres, and M. Sellke were at Microsoft Research.} \\
Stanford University
\and 
Yuval Peres \footnotemark[1]
\and
Mark Sellke \footnotemark[1]\\
Stanford University}
\date{\today}

\maketitle

\abstract{We consider the non-stochastic version of the (cooperative) multi-player multi-armed bandit problem. The model assumes no communication at all between the players, and furthermore when two (or more) players select the same action this results in a maximal loss. We prove the first $\sqrt{T}$-type regret guarantee for this problem, under the feedback model where collisions are announced to the colliding players. Such a bound was not known even for the simpler stochastic version. We also prove the first sublinear regret guarantee for the feedback model where collision information is not available, namely $T^{1-\frac{1}{2m}}$ where $m$ is the number of players.
}

\section{Introduction}
We consider a {\em decentralized/multi-player} version of the adversarial multi-armed bandit problem \citep{ACFS03} and its generalization to multiple plays \citep{UNK10}. Let us first describe the classical centralized version: At each time step $t=1, \hdots, T$, a centralized agent selects a set $S_t \subset [K]$, $|S_t| = m$, of $m$ actions, and simultaneously an adversary selects a loss for each action $\ell_t : [K] \rightarrow [0,1]$. The player's feedback is the set of suffered losses $(\ell_t(a))_{a \in S_t}$ (so-called semi-bandit feedback \citep{ABL14}). The player has access to external randomness, and can select her set of actions $S_t$ based on the history $(S_s, \ell_s(a))_{a \in S_s})_{s<t}$. The agent's perfomance at the end of the game is measured through the 
{\em pseudo-regret} (the expectation is with respect to the randomness in her strategy) :
\[
R_T = \max_{S \subset [K], |S| = m} \E \sum_{t=1}^T \left(\sum_{a \in S_t} \ell_t(a) - \sum_{a \in S} \ell_t(a)\right) \,.
\]
The optimal regret in this centralized setting is known to be $\Theta(\sqrt{K T m})$ \citep{ABL14}. We refer to \cite{BC12, LS19} for more background on bandit problems.
\newline

In this paper we are interested in the {\em decentralized} version of this problem, where there are $m$ independent players chosing the actions instead of a single agent choosing all $m$ actions at once. We assume that each player observes only its own loss, and that there is no communication at all between the players. Moreover when two or more players select the same action at a given round they all get a loss of $1$ instead of the true underlying loss of that action (as well as a signal that a collision occured). This decentralized setting with collision was first introduced, roughly at the same time, in \citet{LJP08, LZ10, AMTS11}, motivated by cognitive radio applications. The no-communication aspect was emphasized in \citet{AM14, RSS16}. More recently an even more challenging setting was proposed, where in case of a collision the players do not even get the information that a collision occured (they only see a loss of $1$) \citet{BBMKP17, LM18, BP18}. All of the works mentioned so far have focused on the classical {\em stochastic} version of the problem \citep{Rob52} where the loss sequence $(\ell_t)_{t \in [T]}$ is assumed to be i.i.d. (in this stochastic setting, the centralized multiple plays problem discussed above go back to \citet{AVW87}). The non-stochastic version that we study here was mentioned as an open problem in \cite{RSS16} with collision information, and in \cite{LM18} without collision information. We note this non-stationary model is particularly appropriate in the context of cognitive radio applications. Very recently a first result for the collision information case was posted on arXiv \citep{ALK19}, with a suboptimal $T^{2/3}$ regret. We note that even with the (much) stronger stochastic assumption, no $\sqrt{T}$-regret is known for this multi-player multi-armed bandit problem.
\newline

In this paper we prove that with collision information the players can actually obtain the optimal $\sqrt{T}$-regret. Furthermore without collision information we propose the first sublinear strategy, although with regret degrading rapidly as the number of player increases, namely $T^{1-\frac{1}{2 m}}$. 
These results are proved for an {\em oblivious} adversary, that is the entire loss sequence $(\ell_t)_{t \in [T]}$ is chosen at the beginning of the game. 
We show that that this assumption is necessary to obtain sublinear regret, that is we prove that an adaptive adversary can induce a worst-case regret of $\Omega(T)$ (even if the players have access to collision information). 
This gap between no non-trivial guarantee for adaptive adversaries and sublinear regret for oblivious adversaries is reminiscent of bandit with switching cost \citep{DDKP14}. However, interestingly, in the latter case the oblivious minimax regret is $\tilde{\Theta}(T^{2/3})$ while here the gap is even more striking as we achieve the optimal $\tilde{O}(\sqrt{T})$ regret against oblivious adversaries.
\newline


For sake of clarity in this preliminary version we primarily focus on the two players case. We briefly discuss the generalization to $m>2$ players with collision information in Section \ref{sec:discussion}, and we give more details for the no-collision case in Section \ref{sec:withoutmany}.

\section{Model and main results} 
We consider two players, Alice and Bob. At each time step $t=1, \hdots, T$, Alice chooses an action $A_t \in [K]$ and Bob chooses $B_t \in [K]$, possibly using external sources of randomness (i.e., uniform random variables in $[0,1]$) $R_t^A$ and $R_t^B$. In addition to the fresh randomness, these actions are chosen based on their respective past history $H_t^A=(R_s^A, \max(\ell_s(A_s), \ds1\{A_s = B_s\}))_{s < t}$ and $H_t^B=(R_s^B, \max(\ell_s(B_s), \ds1\{A_s = B_s\}))_{s < t}$ (note in particular that in the event of a collision, both players observe a loss of $1$). Moreover in the case where collision information is available, we add $\ds1\{A_s = B_s\}$ to the history of both players, e.g., $H_t^A=(R_s^A, \max(\ell_s(A_s), \ds1\{A_s = B_s\}), \ds1\{A_s = B_s\})_{s < t}$.
Whether it is with or without collision information, our notion of regret is as follows:
\[
R_T = \max_{a \neq b, a,b \in [K]} \E \sum_{t=1}^T \bigg(\max(\ell_t(A_t), \ds1\{A_t = B_t\}) + \max(\ell_t(B_t), \ds1\{A_t = B_t\}) - (\ell_t(a) + \ell_t(b))  \bigg) \,.
\]

\subsection{Sources of randomness} \label{sec:randomness}
We consider two models for the external sources of randomness: (a) {\em shared randomness} where $R_t^A = R_t^B$, and (b) {\em non-shared randomness} where $R_t^A$ and $R_t^B$ are independent. Our two core results are for the non-shared randomness model, namely that one can get the classical and optimal $\sqrt{T}$-regret with collision information, and that sublinear regret is actually achievable even without collision information. We introduce the shared randomness model for two reasons. First it is quite natural, especially from the minimax perspective, but also for algorithm design. Indeed one can for example easily get a sublinear regret strategy without collision information but with shared randomness: simply run two versions of Exp3, one for Alice and one for Bob, and couple the draws so as to minimize the number of collisions (we note that such a strategy, while sublinear, cannot possibly achieve $\sqrt{T}$-regret because Alice and Bob's history are diverging too quickly, leading to many collisions). Second (and most important for us), our $\sqrt{T}$-regret strategy is actually more easily described assuming shared randomness. We then show a simple argument (related to pseudorandom generators) to ``derandomize" the shared randomness part by using the collision information.



\subsection{With collision information} \label{subsec:with}
Our main result is the first ever $\sqrt{T}$-regret guarantee (even for the stochastic model\footnote{By making even stronger assumptions in the stochastic model, e.g., bounded gaps or average losses bounded away from $1$, $\sqrt{T}$-regret was derived respectively in \cite{RSS16} and \cite{LM18}.}) for this multi-player multi-armed bandit problem:

\begin{theorem} \label{thm:with}
Consider the model with collision information and no shared randomness. There exists a two players strategy such that against any oblivious adversary one has $R_T = O(K^2 \sqrt{T \log(K) \log(T)})$.
\end{theorem}
We prove the above result in Section \ref{sec:with}.
\newline

The broad strokes of our $\sqrt{T}$-strategy can be summarized as follows:
\begin{enumerate} 
\item Alice plays a low-switching strategy (i.e., Alice changes actions only every $\sqrt{T}$ rounds --roughly--), inspired from the ``shrinking dartboard" strategy of \citet{GVW10}.
\item During a phase where Alice remains constant, Bob plays an algorithm such as Exp3 \citep{ACFS03} on the remaining actions.
\item When Alice decides to switch actions, she first engages in a communication protocol with Bob to sync their histories. Such a communication is easily achieved using the collision information.
\end{enumerate}
Our actual strategy is significantly more complicated than the above summary, and for good reasons as one has to overcome the following obstacles:
\begin{enumerate}
\item First of all the mixing of information between Alice and Bob is absolutely crucial, as it is known that a low-switching strategy with bandit information cannot achieve $\sqrt{T}$-regret \citep{DDKP14}. On the other hand it is also known that typically information from an ``off-policy" distribution cannot be used with Exp3 to obtain a $\sqrt{T}$-regret, see e.g. [Theorem 4.3, \cite{BC12}]. In other words we will need to reason about the {\em joint} distribution of Alice and Bob. Concretely this comes into play to control the {\em variance} of the unbiased estimators, and it will lead to non-trivial joint decision making of Alice and Bob at communication times.
\item Another reason that one needs to argue about the joint distribution of Alice and Bob is that running Exp3 on an adversarially chosen $K-1$ subset (so-called sleeping expert setting \citep{KNS10}) does not typically achieve low-regret even against the second best arm. Thus again the collaboration between Alice and Bob will be crucial here.
\item Next is perhaps the most difficult conceptual point in our work, namely the idea of doing a filtering strategy (as in \citep{GVW10}) with bandit-type information. Indeed the basic filtering idea is to say that if a random action is currently distributed from $p$, and the next target distribution $q$ does not change by more than $(1-\eta)$ multiplicatively, then one can afford to stay put with probability $(1-\eta)$, while still ensuring to be correctly distributed as the next time step (provided that in the event of a switch one resamples from an appropriately modified distribution). However with bandit feedback the next distribution actually depends on the current action, so the filtering argument has to be significantly more involved. This is the part of the argument where assuming shared randomness makes the description much easier.
\item Finally one needs to ``derandomize" the algorithm, that is to explain how to reduce the shared randomness/collision information model to non-shared randomness/collision information.
\end{enumerate}

\subsection{Without collision information}
Next we give the first sublinear regret bound for the case where collision information is not available. The extra difficulty here is that when the players see a loss of $1$, they don't know if the action was truly bad, or if the loss comes from a collision.

\begin{theorem} \label{thm:without}
Consider the model with neither collision information nor shared randomness. There exists a two players strategy such that against any oblivious adversary one has $R_T = \tilde O(KT^{3/4})$.
\end{theorem}

This second result is proved in Section \ref{sec:without}. Broadly speaking the strategy we propose has a similar skeleton as the collision information strategy, namely Alice is a ``slow" player while Bob is a ``fast" player. An important modification is that we now reserve a ``safe" arm for Bob (in the sense that no collision can happen by playing that arm). We also ignore all the intricacies that resulted from sharing information between the players in the collision information case, as it is not clear at all how to implicitly communicate without collision information (this is also why it seems impossible to obtain a $\sqrt{T}$-regret strategy in this setting). The algorithm is summarized as follows:

\begin{enumerate} 
\item Alice plays a low-switching strategy on the subset of arms $\{2, 3, \cdots, K \}$. In particular Alice never plays arm $1$. The low-switching is implemented by playing in blocks. In particular the times at which Alice switches action are known to Bob.
\item During a phase between two switches of Alice, Bob plays an algorithm such as Exp3 \citep{ACFS03} on a growing subset of arms $\mathcal{S}_t$. Initially, at the start $t_0$ of a new phase, $\mathcal{S}_{t_0} = \{1\}$. That is Bob starts by focusing on the safe arm $1$. During the phase Bob will regularly explore the set of Alice's actions $\{2,3,\cdots,K\}$, and when he encounters an arm with loss $<1$ he adds it to his active pool of arms $\mathcal{S}_t$ (indeed Bob now knows that Alice cannot be on this arm, for otherwise the loss would have been $1$ due to collision). On the other hand if exploring an arm always result in a loss of $1$, it means that either Alice is sitting at that action for that phase, or that this arm is actually bad, so Bob does not need to consider it to guarantee low regret.
\end{enumerate}

One difficulty here is Obstacle 2 mentioned in Section \ref{subsec:with}. With collision information we alluded to the fact that this obstacle will be resolved by careful collaboration between Alice and Bob. However in this no-collision information case such collaboration cannot happen. Instead we propose a more sophisticated argument that requires Alice to play a strategy with low {\em internal regret} \citep{Sto05, BM07}. See Section \ref{sec:without} for the details.
\newline

We also explain in Section \ref{sec:withoutmany} how to generalize this approach to $m>2$ players, and we show that in this case the regret worsens to $T^{1- \frac{1}{2m}}$.

\subsection{Adaptive adversaries}
The above results are restricted to oblivious adversaries. This is partially justified by the following result, which shows that, at least without shared randomness, one cannot obtain any non-trivial guarantees.

\begin{proposition} \label{thm:lowerbound}
Let $K=3$. For any two players strategy without shared randomness, there exists an adaptive adversary such that $R_T \geq T / 32$.
\end{proposition}

\begin{proof}
Fix a round $t$, and consider the distributions $p_t^A$ and $p_t^B$ from which $A_t$ and $B_t$ are sampled from (conditionally on respectively $H_t^A$ and $H_t^B$). If there exists a player, say $A$, and an action, say $1$, such that $p_t^A(1) \geq 3/4$, then the adversary can play the loss $\ell_t=(1,0,0)$ which induces a loss for the players of at least $3/4$. On the other hand if for both players all actions have probability less than $3/4$, then it must be that there exists an action $i$ such that  $p_t^A(i) p_t^B(i) \geq 1/64$ (indeed, the probability of both top two actions for both players must be at least $1/8$, and since there are only $3$ actions there must be a common action in their two top actions). In this case the adversary simply plays the loss $\ell_t=(0,0,0)$, which results in an expected loss for the players of at least $1/32$ (coming from the event of a collision). 

Denote $\tau$ for the number of rounds where $\ell_t \neq 0$. Then the player's total loss is at least $3/4 \cdot \tau + (T-\tau) / 32$, while on the other hand it is easy to see that there is a pair of actions whose total loss is at most $\tau \cdot 2/3$.
\end{proof}

\subsection{Open problems} \label{sec:discussion}
A number of questions remain open:
\begin{enumerate}
\item Most intriguing of all is whether one could prove {\em lower bounds} in the most challenging scenario (no collision information, no shared randomness). We believe that in this case the optimal regret with $2$ players is $\Omega(T^{2/3})$. Moreover the exponent could possibly degrade as the number of players increases (indeed our best upper bound in this case is $T^{1-\frac{1}{2m}}$, see Section \ref{sec:withoutmany}).
\item We briefly mentioned in Section \ref{sec:randomness} that with no collision information but with shared randomness one could achieve a sublinear regret. In this case we believe that one can achieve a regret of $O(T^{4/5})$ for any number of players. This points to a significant difficulty for proving that the regret degrades with the number of players in the most challenging scenario. Indeed such an argument would then need to rest on the fact that there is no shared randomness. We note that from a game-theoretic point of view the shared randomness case is the easiest to reason about (as one can think of choosing a single distribution over a profile of $m$ deterministic strategies, and thus the minimax theorem applies between the set of $m$ players and the oblivious adversary).
\item The generalization of Theorem \ref{thm:with} to $m>2$ players does not present any major obstacles, although there are a number of technical complications. We believe that the dependency on $T$ remains optimal (i.e., $\sqrt{T}$) but it is unclear at this point if the dependency on $m$ is polynomial or exponential. We will address this point in the full version of the paper.
\item What can be said about the multi-player version of combinatorial (semi) bandits \citep{CL11,ABL14}?
\end{enumerate}


\section{Proof of Theorem \ref{thm:with}} \label{sec:with}
Our goal here is to present a $\sqrt{T}$-regret strategy for the case with collision information and no shared randomness. To simplify the presentation we first give a $\sqrt{T}$-regret algorithm in a different model, where the players can communicate, and communication happens instantaneously. Moreover we also assume shared randomness.
In this new model, in addition to the standard regret, we will also control the number of times that such communication occurs and the total number of bits exchanged. Our main result then reads as follows:
\begin{theorem} \label{thm:withcom}
There exists a strategy for two players multi-armed bandit with communication and shared randomness such that:
\begin{enumerate}
\item There is no collision at all between the players.
\item The regret is $2^9 K^{3/2} \log(K) \sqrt{T}$.
\item The number of times the players communicate is (in expectation) $K^{3/2} \sqrt{T}$.
\item The number of bits exchanged at each communication step is $O(K \log(T))$.
\end{enumerate}
\end{theorem}

We show in Section \ref{sec:PRG} how to remove the shared randomness assumption in the above result, and then we also show in Section \ref{sec:reduction} how to replace the communication assumption by collision information. This then completes the proof of Theorem \ref{thm:with}. Until Section \ref{sec:PRG} we focus on proving Theorem \ref{thm:withcom}.


\subsection{Notation}
We denote $H_t$ for the randomness up to time $t$, consisting of $(R_s^A,R_s^B)$ for $s\leq t$. We recall that $A_t$ (respectively $B_t$) is the action played by Alice (respectively Bob) at time $t$. We denote by $Q_t$ (respectively $P_t$) the probability distribution of $\{A_t, B_t\}$ (respectively $(A_t, B_t)$) conditionally on $H_{t-1}$. In other words $Q_t$ is the distribution over the {\em unordered pair} of actions $\{A_t, B_t\}$, while $P_t$ is the distribution over {\em ordered pairs} $(A_t, B_t)$, where the ordering simply means that the actions are assigned respectively to Alice and Bob. Finally we denote $p_t^A$ to be the marginal of the first coordinate of $P_t$, and $p_t^B( \cdot | a)$ to be the distribution of the second coordinate of $P_t$ conditionally on the first one being $a$.
\newline

It will be convenient for us to explicitly design some formulas for $Q_t$ and $P_t$, as well as some sampling strategy for $A_t$ and $B_t$. Part of the proof will be to show {\em consistency}, that is that the proposed sampling strategy is such that $Q_t$ and $P_t$ have the meaning ascribed above (namely the distribution of the unordered pair of actions, and the distribution of the ordered pair of actions).

\subsection{Communication and filtering}
We will design a strategy such that $P_t(a,a)=0$ for all time $t$ and all actions $a$. In other words there will never be any collision.
On the other hand at the beginning of each round $t$, with some probability the players will communicate to sync their history (when such communication does occur we refer to $t$ as a \emph{random communication round}). The players will also communicate with probability one at times $t \in \lfloor \sqrt{T/K} \rfloor \cdot \N$ (we refer to such $t$ as \emph{fixed communication rounds}). Moreover we will ensure that $A_t$ remains constant between any two syncing. In fact, the random communication rounds will exactly corresponds to the rounds where $A_t \neq A_{t-1}$. In this way Bob will always know where is Alice, and will always select his action from the remaining arms to avoid collision. Moreover, the probability $P_t$ will be designed such that at each time step Bob can ensure that, conditionally on $A_t$ and $H_{t-1}$ one has:
\begin{equation} \label{eq:needed0}
B_t \sim p_t^B(\cdot | A_t) \,. 
\end{equation}
In other words we will need to ensure that if $t$ is not a communication round then $p_{t + 1}^B( \cdot | A_t)$ does not depend on $\ell_t(A_t)$ (so that Bob does not need to know the loss of Alice's arm to update his probability). This will be proved in Lemma \ref{lem:Bobisfine}.

To ensure that the random communication rounds are not too frequent we need Alice to implement some kind of low switching strategy. Taking inspiration from the ``shrinking dartboard" algorithm \citep{GVW10} we propose the following simple ``filtering" lemma:
\begin{lemma}\label{lem:shrinking}
Let $q$ and $p$ two probability distributions such that $q(i) \geq (1-\epsilon_i) p(i)$. Define the probability distribution $r$ via $r(i) = \frac{q(i) - (1-\epsilon_i) p(i)}{\sum_{j} p(j) \epsilon_j}$. Let $b_i \sim \mathrm{Ber}(\epsilon_i)$, $I \sim p$, and $J \sim r$. Then the distribution of $(1-b_I) I + b_I J$ is $q$.
\end{lemma}
\begin{proof}
The probability that $(1-b_I) I + b_I J$ is equal to some $i$ is given by $(1-\epsilon_i) p(i) + \sum_{j} \epsilon_j p(j) r(i)$, which is equal to $q(i)$ by definition of $r$.
\end{proof}
In words the lemma says that, if Alice's action $A$ is currently distributed from $p$, and if she wants to now be distributed according to $q$ such that $q(i) \geq (1-\epsilon_i) p(i)$ for all $i$, then she can afford to remain on $A$ with probability $1-\epsilon_A$ (provided that when she switches she resamples from an appropriate distribution as indicated in the lemma). A major difficulty in applying this lemma to the bandit setting is that typically the next action distribution $q$ depends on the current action $A$ being played (through the unbiased loss estimator), rendering the lemma all but useless. Such difficulty is to be excepted, as a low switching strategy provably does not exist for single-player multi-armed bandit \citep{DDKP14}. A major conceptual contribution of our work is to leverage the fact that there are {\em multiple players} to go around this difficulty. Namely in the next section we propose a new unbiased estimator based on a shared random bit so that \textbf{Alice's action at time step $t$ is correctly distributed even when conditioned on Alice's distribution at time $t+1$}, or in other words:
\begin{equation} \label{eq:seeyouyesterday}
\P(A_t = a | p_t^A(\cdot), p_{t+1}^A(\cdot)) = p_t^A(a) \,.
\end{equation}
In particular it should be that, given $p_{t+1}^A$ and $p_t^A$, there is some uncertainty remaining on what action Alice played (this sentence is clearly not true for classical single-player multi-armed bandit strategies).
Equipped with \eqref{eq:seeyouyesterday} one can use Lemma \eqref{lem:shrinking} to implement a (time-wise) marginally correct distribution for Alice while also having low-switching (given a control on the multiplicative updates of the distribution, see \eqref{eq:needed1} below). The exact sampling algorithm that we propose is described in Algorithm \ref{alg:sampling}, 
where $L$ is some constant to be defined (we will have $L=O(K)$), and $\Xi_t(a), a \in [K]$ is a carefully chosen set of parameters to be defined later. In particular $\Xi_t(a)$ only depends on the information exchanged at the \emph{fixed communication rounds} as well as the losses observed since then by plays of arm $a$. Moreover $\Xi_t(a)$ will verify:
\begin{eqnarray} \label{eq:needed1}
p_{t+1}^A(a) \geq \left(1- \eta L - \frac{\eta}{\Xi_t(a)} \right) p_t^A(a) \,, \\
\label{eq:needed01}
\Xi_t(a) \geq \frac{1}{4K^2} p_t^A(a) \,.
\end{eqnarray}
We note that thanks to \eqref{eq:needed01}, the expected number of random communication rounds (i.e., times at which Alice switches action) is bounded from above by:
\begin{equation} \label{eq:numberofcomsteps}
\E \sum_{t=1}^T \sum_{a=1}^K p_t^A(a) \left( \eta L + \frac{\eta}{\Xi_t(a)} \right) \leq \eta (4 K^3 + L )T \,.
\end{equation}
Moreover we have the following simple (but crucial) lemma:
\begin{lemma} \label{lem:simplebutkey}
Assuming \eqref{eq:needed0}, \eqref{eq:seeyouyesterday}, and \eqref{eq:needed1} one has that Algorithm \ref{alg:sampling} is consistent in expectation, meaning that for every $a \not= b$ and every $t \in [T]$, one has
\[
\P(A_t = a, B_t = b) = \E[P_t((a, b))] \,,
\]
where the probability and expectation are with respect to the whole history $H_t$.
\end{lemma}

\begin{proof}
The proof is a simple exercise in conditioning and applying Lemma \ref{lem:shrinking}.
\end{proof}

\begin{algorithm} 
    \caption{Two players filtering strategy}\label{alg:sampling}

    \begin{algorithmic}[1]
    \State At a fixed communication round, Alice and Bob communicate their history, then Alice computes $p_t^A(\cdot)$ and picks an arm according to this distribution. Alice tells Bob which arm $A_t$ she picks, Bob then computes $p_t^B(\cdot | A_t) $ and samples $B_t$ according to this distribution. \Comment{Fixed communication round}

    \State At other rounds, with probability $1 - \frac{2 \eta}{\Xi_{t - 1}(A_{t - 1})}$, Alice picks $A_t = A_{t - 1}$, Bob computes $p_t^B(\cdot | A_t) $ and samples $B_t$ according to this distribution. 
    
     \Comment{No communication}
     \State With probability $\frac{2 \eta}{\Xi_{t - 1}(A_{t - 1})}$, Alice and Bob communicate their history, then Alice computes $(p_t^A(\cdot), p_{t - 1}^A(\cdot) )$ and picks each arm according to distribution $r_t$ such that for every $i \in [K]$:
     \[ r_t(i) = \frac{p_t^A(i) - (1 - \eta L - \frac{\eta}{\Xi_{t - 1}(i)}) p_{t - 1}^A(i)}{\sum_{j \in [K]} \left(\eta L + \frac{\eta}{\Xi_{t - 1}(j)} \right) p_{t - 1}^A(j)} \] \Comment{Random communication round}
\State In the second case, Alice then tells Bob which arm $A_t$ she picks, Bob computes $p_t^B(\cdot | A_t) $ and samples $B_t$ according to this distribution.

    \end{algorithmic}
\end{algorithm}

\subsection{Outline of the rest of the proof} \label{sec:outlinerest}
First, in Section \ref{sec:newunbiased}, we propose a new unbiased loss estimator that allows us to ensure \eqref{eq:seeyouyesterday} when $Q_t$ is derived from exponential weights on those estimators, and $P_t$ only depends on those estimators. This result, actually proved in Section \ref{sec:outofideafornames}, will be based on more assumptions on the (currently mysterious for the reader) parameters $\Xi_t(i)$. In Section \ref{sec:outofideafornames} we also work out the variance term for the exponential weights based on these new estimators. Next in Section \ref{sec:variance} we carefully design $P_t$ so as to control this variance term (while also being such that one can satisfy \eqref{eq:needed0}). In particular the parameters $\Xi_t(i)$ are also designed in that section (Section \ref{sec:Xi}). Once the algorithm is fully specified, we prove in Section \ref{sec:verification} the various assumptions we made on the way. A summary of the proof is given in Section \ref{sec:summary}. Finally in Section \ref{sec:PRG} we show how to remove the shared randomness assumption, and in Section \ref{sec:reduction} we show how to relate the communication framework to the collision framework.

\subsection{A new unbiased estimator} \label{sec:newunbiased}
Our basic idea to ensure sufficient randomness in Alice's action at time $t$, even given its distribution at time $t+1$ (in the hope to satisfy \eqref{eq:seeyouyesterday}), is to decide at random whether Alice or Bob records its observed loss at time $t$, leveraging also the shared randomness assumption. Namely $a_t$ and $b_t$ are two random bits such that $a_t + b_t \leq 1$, and our proposed loss estimator $\tilde{\ell}_t$ will have the property that $\tilde{\ell}_t(i) \neq 0$ if and only if: either $i=A_t$ and $a_t=1$, or $i=B_t$ and $b_t=1$. The mean of $a_t$ will be known to both players, namely it is $\Xi_t(A_t)$ (which by construction will only depend on the information exchanged at the fixed communication rounds), while the mean of $b_t$ will be set to ensure \eqref{eq:seeyouyesterday} (see Lemma \ref{lem:critical} below). More precisely, assuming that (yet another constraint on $\Xi_t$ to be verified once it is defined in Section \ref{sec:Xi}, just as \eqref{eq:needed01})
\begin{equation} \label{eq:needed02}
\Xi_t(b) \leq \frac{p_t^B(b | a)}{2}, \forall a \neq b \,,
\end{equation}
we set $a_t \sim \mathrm{Ber}(\Xi_t(A_t))$ and $b_t \sim  \mathrm{Ber}\left(\frac{\Xi_t(B_t)}{p_t^B(B_t | A_t)} \right)$, two \emph{dependent} random variables such that $a_t + b_t \leq 1$. Our unbiased loss estimator is then defined by:
\begin{equation} \label{eq:defestimator}
\tilde{\ell}_t = \frac{\ell_t(A_t) }{\Xi_t(A_t)} e_{A_t} a_t + \frac{\ell_t(B_t)}{\Xi_t(B_t)} e_{B_t}b_t \,.
\end{equation}
One can easily verify the unbiasedness (recall \eqref{eq:needed0}):
\begin{equation} \label{eq:unbiased}
\E[\tilde{\ell}_t| H_{t-1}] = \E_{A_t \sim p_t^A} \left[ \E[\tilde{\ell}_t | A_t, H_{t-1}] \right] = \E_{A_t \sim p_t^A} \left[ \E \left[\ell_t(A_t) e_{A_t} + \frac{\ell_t(B_t)}{p_t^B(B_t |A_t)} e_{B_t} \bigg| A_t, H_{t-1}\right] \right] = \ell_t \,.
\end{equation}
As we explain next, we will want Alice and Bob to play according to the exponential weights distribution on the above loss estimators, while only doing the communication mentioned in the previous section (Algorithm \ref{alg:sampling}).

\subsection{Two players filtering and exponential weights} \label{sec:outofideafornames}
Let us denote $w_t(a) = \exp\left( - \eta \sum_{s <t} \tilde{\ell}_s(a) \right)$. We will design strategies such that for any time $t$ and any actions $a \neq b$, one has:
\begin{equation} \label{eq:Qt}
Q_t(\{a,b\}) \propto w_t(a) w_t(b) \,.
\end{equation}
Note that at this point we still have the flexibility of the {\em assignment procedure}, that is the design of the distribution $P_t$ such that for any $a, b$, $P_t((a,b))+P_t((b,a)) = Q_t(\{a,b\})$. We will design $P_t$ in Section \ref{sec:variance}, based on the variance calculation (Lemma \ref{lem:reg1} below) for the strategy $Q_t$ described by \eqref{eq:Qt}. The important point for us will be that $P_t$ can be calculated with only the knowledge of $\tilde{\ell}_1,\hdots, \tilde{\ell}_{t-1}$. This is quite a non-trivial assumption, in the sense that $P_t$ (and thus presumably also $\Xi_t$) should not depend on the actions actually played by the players (except through their implicit effect on the unbiased loss estimators), and also should not depend on things such as $\tau_c(t)$ (the last random communication round).

\begin{lemma}  \label{lem:critical}
Let us assume \eqref{eq:needed0}, \eqref{eq:needed1}, \eqref{eq:needed02}, and that $P_t$ only depends on $\tilde{\ell}_1,\hdots, \tilde{\ell}_{t-1}$. Then one has \eqref{eq:seeyouyesterday}.
\end{lemma}

\begin{proof}
We will in fact prove by induction that 
\begin{equation} \label{eq:toprovekey}
\P(A_{t} = a | \tilde{\ell}_1, \hdots \tilde{\ell}_{t}) = p_{t}^A(a) \,.
\end{equation}
First note that this is a stronger claim than \eqref{eq:seeyouyesterday}, since by assumption using $\tilde{\ell}_1, \hdots, \tilde{\ell}_{t}$ one can build $P_t$ and $P_{t+1}$, and thus also $p_t^A$ and $p_{t+1}^A$. Moreover note that \eqref{eq:toprovekey} implies $\P(A_{t+1} = a | \tilde{\ell}_1, \hdots, \tilde{\ell}_{t}) = p_t^A(a)$, since Alice is implementing the filtering strategy from Lemma \ref{lem:shrinking}, and conditioning on $\tilde{\ell}_1, \hdots, \tilde{\ell}_t$ fixes both $p_t^A$ and $p_{t+1}^A$. 
We now distinguish three cases to prove \eqref{eq:toprovekey}: $\tilde{\ell}_t(a) \neq 0$, $\tilde{\ell}_t(b) \neq 0$ for some $b \neq a$, and finally $\tilde{\ell}_t = 0$.

\paragraph{Case 1: $\tilde{\ell}_t(a) \neq 0$.} One then has (the second equality is true by induction and \eqref{eq:needed0}):
\begin{align*}
& \P(A_t = a | \tilde{\ell}_1, \hdots \tilde{\ell}_{t}) \\
& = \frac{\P(A_t = a \text{ and } a_t = 1 | \tilde{\ell}_1, \hdots \tilde{\ell}_{t-1})}{\P(A_t = a \text{ and } a_t = 1 | \tilde{\ell}_1, \hdots \tilde{\ell}_{t-1}) + \P(B_t = a \text{ and } b_t = 1 | \tilde{\ell}_1, \hdots \tilde{\ell}_{t-1})} \\
& = \frac{p_t^A(a) \xi_t(a)}{p_t^A(a) \Xi_t(a) + \sum_{a'\neq a} p_t^A(a') p_t^B(a | a') \frac{\Xi_t(a)}{p_t^B(a|a')}} \\
& = p_t^A(a) \,.
\end{align*}

\paragraph{Case 2: $\tilde{\ell}_t(b) \neq 0$ for some $b \neq a$.} One then has (the second equality is true by induction and \eqref{eq:needed0}):
\begin{align*}
& \P(A_t = a | \tilde{\ell}_1, \hdots \tilde{\ell}_{t}) \\
& = \frac{\P(A_t = a \text{ and } B_t = b \text{ and } b_t =1 | \tilde{\ell}_1, \hdots \tilde{\ell}_{t-1})}{\sum_{a' \neq b} \P(A_t = a' \text{ and } B_t = b \text{ and } b_t =1 | \tilde{\ell}_1, \hdots \tilde{\ell}_{t-1}) + \P(A_t = b \text{ and } a_t = 1 | \tilde{\ell}_1, \hdots \tilde{\ell}_{t-1})} \\
& = \frac{p_t^A(a) p_t^B(b |a) \frac{\Xi_t(b)}{p_t^B(b|a)}}{\sum_{a' \neq b} p_t^A(a') p_t^B(b |a') \frac{\Xi_t(b)}{p_t^B(b|a')} + p_t^A(b) \Xi_t(b)} \\
& = p_t^A(a) \,.
\end{align*}

\paragraph{Case 3: $\tilde{\ell}_t = 0$.} For sake of simplicity one can assume that this case only happens when $a_t = b_t = 0$ (indeed one can artifically add $\epsilon$ to all loss values without changing anything). Thus one has (the second equality is true by induction, \eqref{eq:needed0}, and crucially the fact that $a_t$ and $b_t$ are coupled to never be one together):
\begin{align*}
& \P(A_t = a | \tilde{\ell}_1, \hdots \tilde{\ell}_{t}) \\
& = \frac{\sum_{b \neq a} \P(A_t = a \text{ and } B_t = b \text{ and } a_t = 0 \text{ and } b_t = 0 | \tilde{\ell}_1, \hdots \tilde{\ell}_{t-1})}{\sum_{a'} \sum_{b \neq a'} \P(A_t = a' \text{ and } B_t = b \text{ and } a_t = 0 \text{ and } b_t = 0 | \tilde{\ell}_1, \hdots \tilde{\ell}_{t-1})} \\
& = \frac{\sum_{b \neq a} p_t^A(a) p_t^B(b|a) \left(1 - \Xi_t(a) - \frac{\Xi_t(b)}{p_t^B(b|a)} \right)}{\sum_{a'} \sum_{b \neq a'} p_t^A(a') p_t^B(b |a') \left(1 - \Xi_t(a') - \frac{\Xi_t(b)}{p_t^B(b|a')} \right)} \\
& = \frac{p_t^A(a) \left(1 - \sum_b \Xi_t(b) \right)}{\sum_{a'} p_t^A(a') \left(1 - \sum_b \Xi_t(b) \right)} \\
& = p_t^A(a) \,.
\end{align*}
\end{proof}

Next we give the core regret bound for our strategy (of particular importance is the form of the variance term, which will guide us in the construction of $P_t$ in the next section):

\begin{lemma}\label{lem:reg1}
Let $L = \max_{a \neq b, t \in [T]} \frac{p_t^B(b | a) }{\Xi_t(b)} $. Then one has for any actions $a \neq b$,
\begin{align*}
& \E \left[\sum_{t=1}^T \ell_t(A_t)+ \ell_t(B_t) - (\ell_t(a) + \ell_t(b)) \right] \\
& \leq  \frac{2 \log(K)}{\eta} + 8\eta L^2 \sum_{t=1}^T\E_{a \sim p_t^A} \left[  \sum_{b=1, b \neq a}^K  \sum_{a',b'=1}^K \frac{Q_t(\{a',b\} P_t((a,b'))}{P_t((a,b))} \right] \,.
\end{align*}
\end{lemma}

\begin{proof}
Denote $x=e_a + e_b$ and $x_t = \E_{\{A,B\} \sim Q_t} (e_A + e_B)$. The classical exponential weights analysis yields:
\[
\sum_{t = 1}^T (x_t - x) \cdot \tilde{\ell}_t \leq \frac{2 \log(K)}{\eta} + \eta \sum_{t=1}^T \sum_{a', b' =1}^K Q_t(\{a',b'\}) (\tilde{\ell}_t \cdot (e_{a'} + e_{b'}))^2 \,.
\]
On the other hand one has, since $\{A_t, B_t\} \sim Q_t$ in expectation by Lemma~\ref{lem:critical}, it holds:
\begin{eqnarray*}
 \E \left[\sum_{t=1}^T \ell_t(A_t)+ \ell_t(B_t) - (\ell_t(a) + \ell_t(b)) \right] & = &  \E \left[\sum_{t=1}^T (x_t - x) \cdot \ell_t \right] \\
& = &   \E \left[\sum_{t=1}^T (x_t - x) \cdot \tilde{\ell}_t \right] \,,
\end{eqnarray*}
where the second equality uses the tower rule, \eqref{eq:unbiased}, and the fact that $x_t$ is measurable with respect to $H_{t-1}$. Thus we see that it only remains to control:
\begin{eqnarray*}
&&\E [ Q_t(\{a',b'\}) (\tilde{\ell}_t \cdot (e_{a'} + e_{b'}))^2 ] \leq  2 \E [ Q_t(\{a',b'\}) (\tilde{\ell}_t(a')^2+ \tilde{\ell}_t(b')^2) ]\\
& \leq & 4 \left(\E \left[Q_t(\{a',b'\}) \left( \frac{\ds1\{A_t, B_t \in \{a',b'\}\}}{p_t^B(B_t | A_t) \Xi_t(B_t)} \right) \right]+ \E\left[Q_t(\{a',b'\}) \left( \frac{\ds1\{A_t, B_t \in \{a',b'\}\}}{\Xi_t(A_t)^2} \right) \right] \right) \\
& \leq & 4 L^2 \left( \E \left[Q_t(\{a',b'\}) \left( \frac{\ds1\{A_t, B_t \in \{a',b'\}\}}{(p_t^B(B_t | A_t))^2} \right) + Q_t(\{a',b'\}) \left( \frac{\ds1\{A_t, B_t \in \{a',b'\}\}}{(p_t^B(A_t | B_t))^2} \right) \right] \right)
\\
& = & 8 L^2 \E \left[Q_t(\{a',b'\}) \left( \frac{\ds1\{A_t, B_t \in \{a',b'\}\}}{(p_t^B(B_t | A_t))^2} \right) \right] \,.
\end{eqnarray*}
Thus 
\[
\sum_{a',b'=1}^K \E [ Q_t(\{a',b'\}) (\tilde{\ell}_t \cdot (e_{a'} + e_{b'}))^2 ] \leq  8L^2 \E \left[ \frac{Q_t(\{\{a',B_t\}, a' \in [K]\})}{(p_t^B(B_t | A_t))^2} \right] \,.
\]
Now it only remains to see that:
\begin{eqnarray*}
\E_{B_t \sim p_t^B(\cdot | A_t)} \left[ \frac{Q_t(\{\{a',B_t\}, a' \in [K]\})}{(p_t^B(B_t | A_t))^2} \right] & = & \sum_{b \neq A_t} \frac{Q_t(\{\{a',b\}, a' \in [K]\})}{p_t^B(b | A_t)} \\
& = &\sum_{b \neq A_t} \frac{Q_t(\{\{a',b\}, a' \in [K]\}) \cdot p_t^A(A_t)}{P_t((A_t, b))} \\
& = &  \sum_{b \neq A_t} \sum_{a',b'} \frac{Q_t(\{a',b\}) \cdot P_t((A_t,b'))}{P_t((A_t, b))} \,.
\end{eqnarray*}
\end{proof}

\subsection{Controlling the variance} \label{sec:variance}
Our objective is now to design $P_t$ such that (i) $P_t((a,b))+P_t((b,a)) = Q_t(\{a,b\}) \propto w_t(a) w_t(b)$, (ii) one can control
\[
V_t : = \E_{a \sim p_t^A} \left[  \sum_{b=1, b \neq a}^K  \sum_{a',b'=1}^K \frac{Q_t(\{a',b\}) P_t((a,b'))}{P_t((a,b))} \right] \,,
\]
(so that one controls the regret bound from Lemma \ref{lem:reg1}), and (iii) one can verify \eqref{eq:needed0} as well as the assumptions from Lemma \ref{lem:critical} that $P_t$ only depends on $\tilde{\ell}_1, \hdots, \tilde{\ell}_t$.
In this section we focus (ii), and we defer (iii) (which depends on the construction of $\Xi_t$ to Section \ref{sec:Xi}.
\newline

We will denote $Z_t$ for twice the normalization constant of $Q_t$, that is:
\[
Z_t := \sum_{a =1}^K \sum_{b =1, b \neq a}^K w_t(a) w_t(b) \,.
\]

\subsubsection{Naive assignment}
The most basic assignment rule is to simply assign uniformly at random, that is set $P_t((a,b)) = \frac{1}{2} Q_t(\{a,b\}$, or in other words $P_t((a,b)) = \frac{w_t(a) w_t(b)}{Z_t}$. Unfortunately it is easy to see that in this case $V_t$ can be unbounded. Indeed consider a case with $3$ actions, where $w_t(1) \gg w_t(2) = w_t(3)$, and consider the term in $V_t$ with $a'=b'=1, a=2, b=3$, that is: $Q_t(\{1,3\}) P_t((2,1)) / P_t((2,3))$. This term appears with probability $p_t^A(1)$ which is constant, the numerator is also constant, but the denominator is tiny. In fact this issue (the largest weight being much larger than the second largest weight) is the only obstacle to bound the variance, indeed one has with the naive assignment:
\[
\sum_{a',b'=1}^K \frac{Q_t(\{a',b\} )P_t((a,b'))}{P_t((a,b))} = 2 \sum_{a' \neq a, b'\neq b} \frac{w_t(a') w_t(b')}{Z_t} = 2 \frac{\sum_{a' \neq a, b'\neq b} w_t(a') w_t(b')}{\sum_{a, b \neq a} w_t(a) w_t(b)} \,.
\]
Assuming that the weights are ordered ($w_t(1) \geq w_t(2) \geq \hdots$), the largest term in the numerator could be $w_t(1)^2$ (if $a, b \neq 1$) while the largest term in the denominator is $w_t(1) w_t(2)$, so the ratio could be as large as $w_t(1)/w_t(2)$.

\subsubsection{A modified assignment with a dominating arm}
Imagine that we draw at random $\{I_t, J_t\}$ from $Q_t$. As we discussed before, if $1 \not\in \{I_t, J_t\}$ (where again for sake of discussion let us assume that the weights are ordered) then there will be essentially no problem in doing a uniformly random allocation (i.e., set $(A_t, B_t) = (I_t, J_t)$ with probability $1/2$, and $(A_t, B_t) = (J_t, I_t)$ with probability $1/2$). On the other hand if $1$ has been sampled we need to be more careful, and owing to the intuition from the naive assignment calculation we want to assign action $1$ to Alice with higher probability. We simply propose in this case to assign $1$ to Bob with probability $\epsilon \simeq \frac{w_t(2)}{2 w_t(1)}$.
\newline

The above description is only to give intuition, as it ignores the fact that Alice and Bob do not have full knowledge of the weights (in particular they might not know which arm actually has the largest weight). The next lemma describes our actual strategy, based on some assumptions that will need to be verified (just like \eqref{eq:needed0} and \eqref{eq:needed1} still need to be verified).
\begin{lemma} \label{lem:variancecontrol}
Let us assume that we have ordered the arms so that 
\begin{equation} \label{eq:needed2}
w_t(1) \geq \frac{1}{2} \max_{i \in [K]} w_t(i) \,.
\end{equation} 
Let $\epsilon_t$ be such that:
\begin{equation} \label{eq:needed3}
\epsilon_t \in \left[ \frac{1}{4} \frac{\max_{i \neq 1} w_t(i)}{w_t(1)} , \min\left( \frac{\max_{i \neq 1} w_t(i)}{w_t(1)}, \frac{1}{2} \right) \right] \,.
\end{equation}
Then consider the assignment rule defined by $P_t((i,j)) = \frac{1}{2} Q_t(\{i,j\})$ if $1 \not\in \{i,j\}$, $P_t((1,i)) = (1 - \epsilon_t )Q_t(\{1,i\})$, and $P_t((i,1)) = \epsilon_t Q_t(\{1,i\})$. One has:

\[
V_t \leq 64 K \,.
\]
\end{lemma} 

\begin{proof}

For each $a$, $b\neq a$, we can compute that 
\begin{eqnarray*}
p_t^B(b | a) &=& \frac{P_t((a, b))}{\sum_{b' \not= a} P_t((a, b'))}
\end{eqnarray*}

To bound the variance, we have that for each $a$, $b\neq a$, we want to control $V_t(a,b) := \sum_{a',b'=1}^K \frac{Q_t(\{a',b\} P_t((a,b'))}{P_t((a,b))}$ (note that in $V_t$ this term is reweighted by $p_t^A(a)$). We consider three cases. Recall that twice the normalization constant for $Q_t(\{a,b\}) \propto w_t(a) w_t(b)$ is denoted $Z_t := \sum_{a', b' \neq a'} w_t(a') w_t(b')$.

\paragraph{Case $1$: $a=1$.} Note that $P_t((1,b)) \geq \frac{1}{2} Q_t(\{1,b\})$ so that:
\[
V_t(1,b) \leq 2 \sum_{a',b'=1}^K \frac{Q_t(\{a',b\}) Q_t(\{1,b'\})}{Q_t(\{1,b\})} = 4 \sum_{a' \neq b, b' \neq 1} \frac{w_t(a') w_t(b')}{Z_t} \leq 8 \,,
\]
where the inequality follows from the fact that for any $a'$ one has (thanks to \eqref{eq:needed2}):
\begin{equation} \label{eq:case1}
\sum_{b' \neq 1} w_t(b') \leq 2 \sum_{b' \neq a'} w_t(b') \,.
\end{equation}

\paragraph{Case $2$: $b=1$.}
One has:
\[
V_t(a,1) \leq 2 \sum_{a',b'=1}^K \frac{Q_t(\{a',1\}) Q_t(\{a,b'\})}{P_t(\{a,1\})} = 8 \sum_{a' \neq 1, b' \neq a} \frac{w_t(a') w_t(b')}{\epsilon_t \cdot Z_t} \leq \frac{8}{\epsilon_t} \,.
\]
Now observe that in $V_t$ the terms $V_t(a,1)$ only appear with probability $p_t^A(a)$. Invoking Lemma \ref{lem:easyprobacalculation} we obtain:
\[
\E_{a \sim p^A_t} [ V_t(a,1) \ds1\{a \neq 1\} ] \leq 40 K \,.
\]

\paragraph{Case $3$: $a \neq 1, b \neq 1$.}
One has:
\begin{eqnarray*}
\sum_{a',b'=1}^K \frac{Q_t(\{a',b\} P_t((a,b'))}{P_t((a,b))} & = & 4 \sum_{a' \neq b, b' \neq a} \frac{w_t(a')}{w_t(a)} P_t((a,b')) \\
& \leq & 4 \sum_{a' \neq b} \frac{w_t(a')}{w_t(a)} P_t((a,1)) + 4 \sum_{a' \neq b, b' \neq 1} \frac{w_t(a') w_t(b')}{Z_t} \,.
\end{eqnarray*}
The second term in the last display is bounded by $8$ (using \eqref{eq:case1}). On the other hand the first term is upper bounded by (using \eqref{eq:needed3} for the first inequality, and \eqref{eq:case1} for the second inequality):
\[
4 \sum_{a' \neq b} \frac{\epsilon_t \cdot w_t(a') w_t(1)}{Z_t} \leq 4 \sum_{a' \neq b} \frac{w_t(a') \max_{i \neq 1} w_t(i)}{Z_t} \leq 8\,.
\]
\end{proof}

\begin{lemma} \label{lem:easyprobacalculation} 
With the assignment rule described in Lemma \ref{lem:variancecontrol} and assuming \eqref{eq:needed2}, \eqref{eq:needed3}, one has
\[
p_t^A(1) \geq 1- 5K \epsilon_t \,.
\]
\end{lemma}

\begin{proof}
By definition of $p_t^A$ and our choice of $P_t((1, b))$, we know that (by $\sum_{b \not= a} P_t((a, b)) = 1$)
\[
p_t^A(1) = \sum_{b \not= 1} P_t((1, b)) = \sum_{b \not= 1} (1 - \epsilon_t) Q_t(\{1, b\})\,.
\]
Next observe that
\[
\sum_{b \neq 1} Q_t(\{1,b\}) = \frac{\sum_{b \neq 1} w_t(1) w_t(b)}{\sum_{a'} \sum_{b' >a'} w_t(a') w_t(b')} = \frac{1}{1+\sum_{a'\neq 1} \sum_{b' >a'} \frac{w_t(a')}{w_t(1)} w_t(b') / \sum_{b \neq 1} w_t(b)} \geq \frac{1}{1+ 4 \epsilon_t K} \,,
\]
where the inequality follows from \eqref{eq:needed3}. Thus we obtain
\[
p_t^A(1) \geq \frac{1- \epsilon_t}{1+4 \epsilon_t K} \geq 1 - 5 K \epsilon_t \,.
\]
%
%

\end{proof}

\subsubsection{The $\epsilon_t$ and $\Xi_t(i)$ parameters} \label{sec:Xi}
To fully specify our algorithm it only remains to define the parameter $\epsilon_t$ for the assignement rule described in Lemma \ref{lem:variancecontrol}, as well as the parameters $\Xi_t(i)$ that were used crucially in the definition of the loss estimators. First, to simplify notation, we reorder the arms at every fixed communication round so that arm $1$ has the largest weight $w_t$, and arm $2$ has the second largest weight. In other words at any time $t$ we have $w_{\tau(t)}(1) = \max_{i \in [K]} w_{\tau(t)}(i)$ and $w_{\tau(t)}(2) = \max_{i \neq 1} w_{\tau(t)}(i)$  (recall that $\tau(t)$ denotes the last fixed communication round before time $t$). We will now use the following formulas:
\begin{equation} \label{eq:defepsilon}
\epsilon_t := \frac{w_{\tau(t)}(2)}{2 w_{\tau(t)}(1)} \,,
\end{equation}
and
\begin{equation} \label{eq:defXi}
\Xi_t(i) := \frac{w_t(i)}{2 K w_{\tau(t)}(2)} \text{ for } i \not= 1 \text{ and } \Xi_t(1) = \frac{1}{8K} \,.
\end{equation}

\subsection{Verifying all the assumptions} \label{sec:verification} 
Now that we have a complete description of the players' strategies we will verify all the assumptions made in the previous sections. It will be useful to first work out the formulas for $p_t^B(b|a)$. 

\begin{lemma} \label{lem:ptBformulas}
One has $p_t^B(b|1) \propto w_t(b)$ for any $b \neq 1$. For any $a  \neq 1$ one has 
\[
p_t^B(b | a) \propto w_t(b) \ds1\{b \neq 1\}  + \epsilon_t w_t(1) \ds1\{b = 1\} \,,
\]
for $b \neq a$.
\end{lemma}
\begin{proof}
First we have:
\[
p_t^B(b|1) = \frac{P_t((1,b))}{\sum_{b' \neq 1} P_t((1,b'))} = \frac{Q_t(\{1,b\})}{\sum_{b' \neq 1} Q_t(\{1,b'\})} = \frac{w_t(b)}{\sum_{b' \neq 1} w_t(b')} \,.
\]
Next we have:
\[
p_t^B(1|a) = \frac{P_t((a,1))}{\sum_{b' \neq a} P_t((a,b'))} = \frac{\epsilon_t Q_t(\{a,1\})}{\epsilon_t Q_t(\{a,1\}) + \sum_{b' \neq a, 1} \frac{1}{2} Q_t(\{a,b'\})} = \frac{\epsilon_t w_t(1)}{\epsilon_t w_t(1) + \frac{1}{2} \sum_{b' \neq a, 1} w_t(b')} \,.
\]
Finally if both $a$ and $b$ are distinct from $1$:
\[
p_t^B(b|a) = \frac{P_t((a,b))}{\sum_{b' \neq a} P_t((a,b'))} = \frac{\frac{1}{2} Q_t(\{a,b\})}{\epsilon_t Q_t(\{a,1\}) + \sum_{b' \neq a, 1} \frac{1}{2} Q_t(\{a,b'\})} = \frac{\frac{1}{2} w_t(b)}{\epsilon_t w_t(1) + \frac{1}{2} \sum_{b' \neq a, 1} w_t(b')} \,.
\]
\end{proof}

\subsubsection{Sampling assumptions} 
We start with \eqref{eq:needed0}, namely that Bob can sample from $p_t^B( \cdot |A_t)$ using only the information received at communication rounds (both fixed and random), as well as his own feedback (in other words $p_t^B( \cdot |A_t)$ should not depend on $\ell_s(A_t)$ for $s \in [\tau_c(t), t]$, where we recall that $\tau_c(t)$ is the last communication round). We also verify that Alice can be implement the filtering by showing that $\Xi_t(A_t)$ similarly only depends on the information available to Alice at round $t$.
\begin{lemma}\label{lem:Bobisfine}
For every $t \in [T]$, we have that $p_t^B( \cdot |A_t)$ 
only depends on $A_{t}, H_{\tau_c(t) }$ and $(B_s, \ell_s(B_s) )$ for every $s \in (\tau_c(t), t)$, but not $\ell_s(A_s)$ for any $s \in (\tau_c(t), t)$. Moreover $\Xi_t(A_t)$ only depends on $H_{\tau_c(t)}$ and $\ell_s(A_t)$ for every $s \in (\tau_c(t), t)$, but not on $(B_s, \ell_s(B_s))$ for any $s \in (\tau_c(t), t)$
%
\end{lemma}

\begin{proof}
Let us prove the first claim by induction. Given Lemma \ref{lem:ptBformulas} it clearly suffices to show that $w_t(b)$ for any $b \neq A_t$ can be computed with such limited information, which in turn only requires $\tilde{\ell}_t(b)$ to be computable with such information. This in turn is clearly true by induction (recall the formulas \eqref{eq:defestimator} and \eqref{eq:defXi}). The second claim is proved similarly.
\end{proof}

Next we also show that Alice's sampling satisfies the bounded multiplicative update given in \eqref{eq:needed1}

\begin{lemma} \label{lem:Aliceisfine}
\eqref{eq:needed1} holds true between each fixed communication rounds.
\end{lemma}

\begin{proof}
By definition, we know that 
\[
p_{t + 1}^A( a) = \sum_{b \not= a} P_{t + 1} ((a, b)) = \frac{\sum_{b \not= a} c_{a, b} w_{t + 1}(a) w_{t + 1}(b) }{Z_{t + 1}} \,,
\]
where $c_{a, b} \in \left\{ \epsilon, 1 - \epsilon, \frac{1}{2} \right\}$ as given in Lemma~\ref{lem:variancecontrol}. 

Recall that $\tilde{\ell}_t$ is non-negative, thus $w_t(a) = \exp\left( - \eta \sum_{s <t} \tilde{\ell}_s(a) \right) $ is non-increasing at every iteration. Which implies that $Z_{t + 1} \leq Z_t$. 
Thus,
\begin{eqnarray*}
p_{t + 1}^A( a) &\geq& \frac{\sum_{b \not= a} c_{a, b} w_{t + 1}(a) w_{t + 1}(b) }{Z_{t}}
\\
&\geq&\frac{\sum_{b \not= a}  e^{- \eta \left( \frac{1}{\Xi_t(a)} +\frac{1}{\Xi_t(b)} \right) }c_{a, b} w_{t }(a) w_{t }(b) }{Z_{t}}
\\
& = &e^{- \frac{\eta}{\Xi_t(a)} } \sum_{b \not= a}p_{t }^A(a) \sum_{b \not= a}p_{t}^B(b | a)e^{- \frac{\eta}{\Xi_t(b)} }
\\
&\geq& e^{- \frac{\eta}{\Xi_t(a)} } p_{t }^A(a) \sum_{b \not= a}p_{t}^B(b | a)\left( 1 - \frac{\eta}{\Xi_t(b)} \right)
\\
&\geq& (1 - \eta L) e^{- \frac{\eta}{\Xi_t(a)} } p_{t }^A(a)
%
%
\end{eqnarray*}
\end{proof}

\subsubsection{Assumptions on $\epsilon_t$} \label{sec:epsilonassumption}
Next we prove that we weights do not change too rapidly. In particular the following lemma easily implies that \eqref{eq:needed2} and \eqref{eq:needed3} holds true.

\begin{lemma} \label{lem:weightsnotmoving}
Assume that $\eta \leq \frac{1}{8 L \sqrt{K T}}$. Then one has
\begin{eqnarray*}
\frac{w_t(1)}{w_{\tau(t)}(1)} \geq \frac{1}{2} \quad \text{and} \quad \frac{w_t(2)}{w_{\tau(t)}(2)} \geq \frac{1}{2} 
\end{eqnarray*}
\end{lemma}

\begin{proof}
We are going to prove this by induction.  Note first that, by definition of $L$ (see Lemma \ref{lem:reg1}) one has for any $a \neq b$, $w_{t+1}(b) \geq w_{t}(b) \exp\left(- \eta L \frac{1}{p_t^B(b | a)} \right)$. We will now show (using the induction hypothesis) that $p_t^B(2 | 1) \geq 1/(4K)$ and $p_t^B(1 | 2) \geq 1/(4K)$ which easily concludes the proof. Indeed the multiplicative change on say $w_t(1)$ compared to $w_{\tau(t)}(1)$ is at most $\exp\left(-(t-\tau(t)) 4 \eta L K \right)$, and $t- \tau(t) \leq \sqrt{T/K}$ by definition of the fixed communication rounds.
\newline

Using Lemma \ref{lem:ptBformulas}, the induction hypothesis, and the definition of the arm ordering (recall Section \ref{sec:Xi}) we have:
\[
p_t^B(2|1) = \frac{w_t(2)}{\sum_{b \neq 1} w_t(b)} \geq \frac{\frac{1}2 w_{\tau(t)}(2)}{\sum_{b \neq 1} w_{\tau(t)}(b)} \geq \frac{1}{2 K} \,.
\]
Similarly we get (recall the definition of $\epsilon_t$ \eqref{eq:defepsilon}):
\begin{equation} \label{eq:reused}
p_t^B(1|2) = \frac{\epsilon_t w_t(1)}{\epsilon_t w_t(1) + \sum_{b \neq 1,2} w_t(b)} \geq \frac{\frac{1}{2} \epsilon_t w_{\tau(t)}(1)}{\epsilon_t w_{\tau(t)}(1) + \sum_{b \neq 1,2} w_{\tau(t)}(b)} \geq \frac{1}{4 K} \,,
\end{equation}
which concludes the proof.
\end{proof}

\subsubsection{Assumptions on $\Xi_t(a)$} \label{sec:Xiassumption}
Finally we conclude the proof by proving the assumptions \eqref{eq:needed01} and \eqref{eq:needed02} on $\Xi_i(t)$, as well as showing that $L=O(K)$. We start with the following result, which directly shows \eqref{eq:needed02} as well as $L \leq 8 K$.

\begin{lemma}\label{lem:var1}
For any $a \neq 1$ we have $p_t^B(1|a) \geq \frac{1}{4K}$. Moreover for any $a$ and $b \not\in \{1,a\}$ we have:
\[
p_t^B(b| a) \in \left[ \frac{w_t(b)}{K w_{\tau(t)}(2)},  \frac{4 w_t(b)}{w_{\tau(t)}(2)}\right] \,.
\]
\end{lemma}
\begin{proof}
The first inequality is proved exactly as \eqref{eq:reused}. For the second statement we distinguish two cases, whether $a=1$ or not.

\paragraph{Case 1: $a=1$.} By Lemma \ref{lem:ptBformulas} we have:
\[
p_t^B(b|1) = \frac{w_t(b)}{\sum_{b' \neq 1} w_t(b')} \geq \frac{w_t(b)}{K w_{\tau(t)}(2)} \,.
\]
For the upper bound we use that $\sum_{b' \neq 1} w_t(b') \geq w_t(2) \geq \frac{1}{2} w_{\tau(t)}(2)$ by Lemma \ref{lem:weightsnotmoving}.

\paragraph{Case 2: $a \neq 1$.} By Lemma \ref{lem:ptBformulas} we have (for $b \neq 1$)
\[
p_t^B(b | a) = \frac{w_t(b)}{\epsilon_t w_t(1) + \sum_{b' \neq 1, a} w_t(b')} \geq \frac{w_t(b)}{K w_{\tau(t)}(2)} \,.
\]
For the upper bound we use that, by Lemma \ref{lem:weightsnotmoving}, $\epsilon_t w_t(1) + \sum_{b' \neq 1} w_t(b') \geq \frac{1}{2} \epsilon_t w_{\tau(t)}(1) = \frac{1}{4} w_{\tau(t)}(2)$.

\end{proof}

\begin{lemma}
For every action $a \in [K]$, $\Xi_t(a)  \geq \frac{1}{4 K^2} p_t^A(a)$ (that is \eqref{eq:needed01} holds true). 
\end{lemma}

\begin{proof}
For $a = 1, 2$ this claim is trivially true. For $a \not= 1, 2$, we have that 
\begin{eqnarray*}
p_t^A(a) &=& \sum_{b \not= a} P_t((a, b)) \leq \frac{2 K w_t(1) w_t(a)}{w_t(1) w_t(2)} \leq \frac{2 K w_t(a)}{w_{\tau(t)}(2)} = 4K^2 \Xi_t(a) \,.
\end{eqnarray*}
\end{proof}

\subsection{Proof summary} \label{sec:summary}
We detail here how to put together the previous sections to obtain Theorem \ref{thm:withcom}. First of all, as indicated by Lemma \ref{lem:Bobisfine}, we know that Alice can compute $\Xi_{t}(A_t)$ and Bob can compute $p_t^B$ (and $\Xi_t(a)$ for $a \neq A_t$) between each communication rounds, so the proposed strategy can indeed be implemented (moreover the assumptions needed on $\Xi_t$ are verified in Section \ref{sec:Xiassumption}). When Alice and Bob Communicates, Alice send Bob $ \sum_{s \in (\tau_c(t), t)}\frac{\ell_s(A_s) }{\Xi_s(A_s)} e_{A_s} a_s $ and Bob sends Alice $\sum_{s \in (\tau_c(t), t)} \frac{\ell_s(B_s)}{\Xi_s(B_s)} e_{B_s}b_s $, each requiring $O(K \log(T))$ bits\footnote{To be more precise Alice and Bob communicates an approximation to these numbers at the $1/\mathrm{poly}(T)$ scale. This does not have any effect on the bounds, so we ignore this minor point.}. Alice also communicates her new action to Bob.
\newline

Since we can take $L=8K$ (see Section \ref{sec:Xiassumption}) we get from \eqref{eq:numberofcomsteps} that the expected number of communication rounds is less than $5 K^3 \eta T$. We also have that the number of fixed communication rounds is less than $\sqrt{T K}$.
\newline

Next we invoke Lemma~\ref{lem:reg1} with Lemma~\ref{lem:variancecontrol} (note that the assumptions in the latter lemma are proved in Section \ref{sec:epsilonassumption}) to obtain that the regret of Alice and Bob is bounded from above by:
\[
\frac{\log(K)}{\eta} + 2^{15} K^3 \eta T \,,
\]
where we have the constraint that $\eta \leq \frac{1}{2^6 \sqrt{K^3 T}}$ from Lemma \ref{lem:weightsnotmoving}.
\newline

Finally, taking $\eta = \frac{1}{2^{7} \sqrt{K^3 T}}$ one obtains a regret of $2^9 K^{3/2} \log(K) \sqrt{T}$ and a total number of communication rounds of $K^{3/2} \sqrt{T}$.

\subsection{Removing the shared randomness} \label{sec:PRG}
For $\ell \in \{0,1\}^{K T}$ and $s \in \{0,1\}^T$, let us denote $\cR_{\ell}(s)$ for the regret suffered by Alice and Bob against the loss sequence $\ell$ when using the bit string $s$ as their shared randomness (recall that the strategy described above only needs one shared random bit per step, to decide who will record their observed loss). More precisely $\cR_{\ell}(s)$ denotes the {\em expected} regret, where the expectation is taken with respect to everything except the shared random bit string $s$. Our proof so far showed that:
\[
\forall \ell \in \{0,1\}^{K T}, \ \E_{s \sim \mathrm{unif}(\{0,1\}^T)} \cR_{\ell}(s) = O_K(\sqrt{T}) \,.
\]
If the shared bit string $s$ was of smaller length, say $O(\sqrt{T})$, then one could remove the shared randomness assumption since $s$ could simply be sampled by say Alice, and then communicated to Bob. Viewing random bits as a resource is the appanage of the theory of pseudorandom generators (see e.g., \citep{Gol10}). Instantiated in our framework, we would like to use a much shorter bit string $s' \in \{0,1\}^{O(\log(T))}$, together with an appropriate map $G: \{0,1\}^{O(\log(T))} \rightarrow \{0,1\}^T$, such that for all $\ell \in \{0,1\}^{K T}$ and all $t \in [0,T] \cap \N$ one has 
\begin{equation} \label{eq:PRG}
 \P_{s' \sim \mathrm{unif}(\{0,1\}^{O(\log(T))}} (\cR_{\ell}(G(s')) \in [t, t+1]) \leq \frac{1}{T}+\P_{s \sim \mathrm{unif}(\{0,1\}^T)} (\cR_{\ell}(s) \in [t,t+1]) \,.
\end{equation}
Note that the above condition directly implies that replacing a truly random $T$-bit string $s$ by $G(s')$ in the algorithm only cost an additive constant $2$ in the regret. Moreover one can assume access to $G(s')$ for both players without assuming shared randomness. Indeed $G$ is a fixed map built once and for all (more on that below), and $s'$ is small enough that it can be communicated at the start of the game. Thus proving \eqref{eq:PRG} is enough to remove the shared randomness assumption (note that we assume here that the losses are taking value in $\{0,1\}$ instead of $[0,1]$, but it is well-known how to reduce the latter to the former).

The map $G$ is usually referred to as a PRG (pseudorandom generator) that {\em fools} the boolean test functions $s \mapsto \ds1\{\cR_{\ell}(s) \in [t,t+1]\}$, for $\ell \in \{0,1\}^{K T}$ and $t \in [0,T] \cap \N$. It is well-known that one can fool $N$ test functions, up to a uniform error of $\epsilon$ in the probabilities, using $s'$ of length only $O(\log\log(N) + \log(1/\epsilon))$ (note that we take $\epsilon = 1/T$ and $N = 2^{KT}$, so we indeed obtain that $s'$ is of length $O(\log(T))$). In fact with a simple Hoeffding's inequality one can show that a random map $G$ works with high probability and in expectation (see e.g., [Exercise 1.3, \citep{Gol10}]). Note that the random map $G$ can be known to the oblivious adversary, so we do not need to communicate $G$ during the game.

\subsection{From communication to collision} \label{sec:reduction}

Finally we describe in Algorithm~\ref{alg:com} the reduction from Algorithm~\ref{alg:sampling} to an algorithm that uses only collision information instead of explicit communication (note that there is an overhead of $O(K)$, namely to communicate one bit there will be $O(K)$ collisions in expectation). This (together with Section \ref{sec:PRG}) completes the proof of Theorem~\ref{thm:with}. Indeed, in Algorithm~\ref{alg:sampling} there is no collision at all between Alice and Bob. Thus, when Bob (Alice) finds a collision, he (she) knows that Alice (Bob) wants to communicate. The expected regret caused by this protocol is $O(K)$ times the expected number of bits of communication. We obtain the claimed dependency on $K$ in Theorem \ref{thm:with} by doing a slightly different optimization on $\eta$ from the one in Section \ref{sec:summary}. 

\begin{algorithm}
    \caption{Communication to Collision}
        \label{alg:com}
    \begin{algorithmic}[1]
    \Require Alice wants to send a bit $s \in \{1, 2\}$ to Bob:
    \While{No Collision}
     \State Alice pick an arm uniformly at random from $K$.
     \EndWhile
     \State Alice pick arm $s$. 
     \State Bob pick arm $1$. 
     \State If no collision, then Bob knows that the bit is $2$, otherwise the bit is $1$. 
    \end{algorithmic}
    \end{algorithm}

\section{Proof of Theorem \ref{thm:without}} \label{sec:without}
Our strategy relies on the notion of swap regret, which we recall in Section \ref{sec:swap}. We then explain both Alice (Section \ref{sec:Alicenocollision}) and Bob's algorithm (Section \ref{sec:Bobnocollision}), and conclude the section with the proof of Theorem \ref{thm:without} (Section \ref{sec:proofnocollision}).

We denote $\hat{\ell}_t(i) = \max(\ell_t(i), \ds1\{A_t = B_t\})$, that is the effective loss functions for the players at round $t$. We also partition $[T]$ into $R $ blocks $\mathcal{B}_1, \cdots, \mathcal{B}_{R}$, where each $\mathcal{B}_r$ is given as: 
\[ 
\mathcal{B}_r= \left\{ \frac{T}{R}(r - 1) + 1,\frac{T}{R}(r - 1) + 2, \cdots, \frac{T}{R}r \right\} \,.
\]

\subsection{Swap regret} \label{sec:swap}
The swap regret of a single-player multi-armed bandit strategy is defined as follows, \citep{Sto05, BM07}, 
\[
\max_{\Phi : [K] \rightarrow [K]} \E \sum_{t=1}^T  \big( \ell_t(A_t) - \ell_t(\Phi(A_t)) \big) \,.
\]
We will not need the full power of swap regret, and in fact it is enough for us to compete against strategies of the form: $\Phi_{a,b}(i) = a$ if $i \neq b$ and $\Phi_{a,b}(b) = b$. However for sake of clarity of exposition we stick with the general swap regret. 
\begin{theorem}[\cite{Sto05}] \label{thm:swap}
There exists a single-player multi-armed bandit strategy with swap regret $O(K \sqrt{T \log(K)})$.
\end{theorem}

\subsection{Alice's algorithm} \label{sec:Alicenocollision}
Alice will restrict her attention to actions $\{2,\hdots, K\}$. Moreover she will view a block $\cB_r$ as a single round (in other words she plays a constant action on a block $\cB_r$). The strategy she uses over those $R$ rounds and $K-1$ actions is the no-swap regret algorithm from Theorem \ref{thm:swap}. We thus obtain the following guarantee: 
\begin{lemma} \label{lem:Alicenocollision}
Alice satisfies $A_t \neq 1$ for all $t$, and:
\[
\max_{\Phi : \{2, \hdots, K\} \rightarrow \{2, \hdots, K\}} \E \sum_{t=1}^T \big( \hat{\ell}_t(A_t) - \hat{\ell}_t(\Phi(A_t)) \big) = O\left(K T \sqrt{\frac{\log(K)}{R}} \right) \,.
\]
\end{lemma}

\begin{proof}
One can apply Theorem \ref{thm:swap} with $T$ replaced by $R$, and the range of the losses is $T/R$ instead of $1$.
\end{proof}

\subsection{Bob's algorithm} \label{sec:Bobnocollision}
Bob restarts his algorithm at the beginning of each block $\cB_r$. During a block, Bob keeps an active set of arms, and plays anytime-Exp3 restricted to these arms. The active set is initialized to $\{1\}$. During the block, with probability $\sqrt{K R / T}$, Bob selects a random action outside of the current active set. If on such exploration rounds Bob observes a loss $<1$, then he adds the explored arm to his active set, and starts a new instance of anytime-Exp3 on this set of active arms. 

\begin{lemma} \label{lem:Bobnocollision}
Let $A_r$ be the action that Alice plays during block $\cB_r$. Then Bob satisfies $B_t = A_r$ with probability at most $\sqrt{K R / T}$. Furthermore for any $b^*_r \neq A_r$,
\[
\E \sum_{t \in \cB_r} \big( \hat{\ell}_t(B_t) - \hat{\ell}_t(b^*_r) \big) \leq O \left( K \sqrt{\frac{T  \log(K)}{R}} \right) \,.
\]
\end{lemma}

\begin{proof}
We first notice that $\hat{\ell}_t(A_r) = 1$ for all $t \in \mathcal{B}_r$, thus $A_r$ is never in the active set, which also implies that $B_t = A_r$ with probability at most $\sqrt{K R / T}$.

Let us now denote $\mathcal{L}_0$ to be the set of all $t \in \mathcal{B}_r$ such that $\ell_t(b^*_r) \not= 1$ and $\mathcal{L}_1$ be the set of all $t \in \mathcal{B}_r$ such that $\ell_t(b^*_r) = 1$. Let us define \[ \mathcal{L}_0(t) = \left|[t] \cap \mathcal{L}_0\right| \,.\]

Let $t^*$ be the time that $b^*_r$ is added to the active set. Note that the active set changes at most $K$ times during each block, hence partitions $\mathcal B_r$ into $\mathcal B_{r,1}\cup\mathcal B_{r,2}\cup\dots\cup\mathcal B_{r,K}$. Observe that if we condition on the set of times that Bob plays within the active set, the conditional law of his play on that subset of times is exactly anytime-Exp3 - the exploration steps occur independently of Bob's performance. Therefore the standard regret bound $O\left(\sqrt{|\mathcal B_{r,i}|K\log K}\right)$ applies on the non-exploration times. Summing over the $K$ Exp3 instances and including the loss from $\mathcal L_0(t^*)$ and from exploration rounds themselves, we have  
\begin{eqnarray*}
 \E\left[\sum_{t \in \mathcal{B}_r} \hat{\ell}_t(B_t) \right] - \sum_{t \in \mathcal{B}_r} \hat{\ell}_t(b^*_r) \leq \E[\mathcal{L}_0(t^*) ]+O\left(\sqrt{\frac{TK}{R}}\right)+ O\left(\mathbb E\left[  \sum_i \sqrt{|\mathcal B_{r,i}|K\log K} \right] \right) \,.
\end{eqnarray*}

Note that Bob samples each arm with probability at least $\sqrt{R / (T K)}$ at each iteration, which implies that $\E[\mathcal{L}_0(t^*) ] \leq O\left( \sqrt{T K / R} \right)$. We have $\sum_{i}|\mathcal B_{r,i}|=\frac{T}{R}$ so by Jensen's inequality, the last term is maximized when $|\mathcal B_{r,i}|=\frac{T}{RK}$ for all $i$. Hence the last term is at most $O\left(K\sqrt{\frac{T \log(K)}{R}}\right).$ This completes the proof. 


\end{proof}

\subsection{Proof} \label{sec:proofnocollision}
First we note that the regret bounds in Lemma \ref{lem:Alicenocollision} and Lemma \ref{lem:Bobnocollision} in fact hold with $\ell_t$ instead of $\hat{\ell}_t$, with an added term $O(\sqrt{K R  T})$ in the former case. Indeed by Lemma \ref{lem:Bobnocollision} the probability of collision is at most $\sqrt{K R / T}$ so the total number of collisions is $\sqrt{K R  T}$.

We now consider any two distinct actions $a , b \in [K]$ such that $a \not= 1$ and show that Alice and Bob achieve small regret against this pair. Let us define a function $f: \{2, 3, \cdots, K\} \to  \{2, 3, \cdots, K\} $ such that $f(i) = a$ for every $i \not= b$ and $f(b) = b$ if $b \in  \{2, 3, \cdots, K\} $. Then, the low swap regret property of Alice ensures that:

\[
\E \sum_{t=1}^T \big( \ell_t(A_t) - \ell_t(f(A_t)) \big) \leq O \left( K T \sqrt{\frac{\log(K)}{R}} + \sqrt{K R  T} \right) \,.
\]

Next, let $g: [K] \to \{ a, b\}$ be a function such that $g(i) = b$ if $i \not= b$ and $g(b) = a$. In particular $g(i) \not= i$ for every $i \in [K]$, then Lemma~\ref{lem:Bobnocollision} ensures that for every $r \in [R]$ and every $A_r$,
\[
\E \sum_{t \in \mathcal{B}_r} \big( \ell_t(B_t) - \ell_t(g(A_r)) \big) \leq O\left( K\sqrt{\frac{T \log K}{R}} \right) \,.
\]

Summing up the above two displays (the second being summed also over all $r\in [R]$) we have:
\begin{eqnarray*}
\E\left[ \sum_{r \in [R]}  \sum_{t \in \mathcal{B}_r} {\ell}_t(A_r) + {\ell}_t(B_t) \right] &\leq& \E\left[\sum_{r \in [R]}  \sum_{t \in \mathcal{B}_r} (\ell_t(f(A_r))  + \ell_t(g(A_r))) \right] 
\\
&&+O \left(K T \sqrt{\frac{\log(K)}{R}} + K\sqrt{T R  \log K} \right)  \,.
\end{eqnarray*}

Note that $\{f(A_r), g(A_r) \} = \{a, b\}$. Since $a$ and $b$ were arbitrary, the final RHS is an upper bound for the expected regret. It remains to optimize over $R$ to obtain $O\left(K T^{3/4} \log^{1/2} K\right).$

\begin{remark}

It is crucial to the above analysis that Alice has low swap regret, not just low regret. If Alice is only guaranteed low regret, even a perfect Bob player might not be able to obtain sublinear regret as a pair. Here we give a simple example to illustrate this point. Consider a game with three actions and $T$ rounds with losses

\[\vec\ell_t =
\begin{cases} 
(0,1,1), 0<t\leq T/3 \\ 
(1,0,1), T/3<t\leq 2T/3\\
(0,0,1), 2T/3<t\leq T
\end{cases}.\]

Suppose that Alice plays action $1$ for the first third, then action $2$ for the next third, and then action $3$ in the final third. Her total loss is $T/3$, so she has $0$ regret. However, given that Alice plays this way, there is no sequence of actions for Bob achieving less than $T/3$ regret. Indeed, Bob always has loss at least $2T/3$ for a total loss of $T$ between Alice and Bob, while the first two actions have total loss $2T/3$. 

\end{remark}

\section{Extension to Many Players with No Collision Information} \label{sec:withoutmany}
Here we extend our analysis without collision information to the case of $m>2$ players, showing a $T^{1-\frac{1}{2m}}$ type regret bound. 

\begin{theorem}
\label{thm:manyplayer}

Let $m\leq K$ and consider the $m$ player bandit game with neither collision information nor shared randomness. There exists an $m$ player strategy such that against any oblivious adversary one has \[R_T=\tilde O\left(mK^{3/2}T^{1-\frac{1}{2m}}\right).\]

\end{theorem}

The proof is given in subsection~\ref{subsec:manyplayerproof}. In preparation, we first describe the algorithm and then give two lemmas. The first lemma essentially controls the swap regret of each player. The second lemma generalizes our functions $f,g$ in the previous section to the $m$ player case.

\subsection{Algorithm description}

The algorithm is similar to the $m=2$ case, with each player using blocks of a different size. We label the players $1$ through $m$, with Player $1$ playing in the largest blocks. More precisely, Player $i$ plays a fixed action on each block $\mathcal B^i_r$ of length $|\mathcal B^i_r|=T^{1-\frac{i}{m}}$. We will denote by $\mathcal B^i_r$ the $r$th such block, for $r\leq T^{i/m}$. For $j\leq i$, we will denote by $A_{j,\mathcal B^i_r}$ the fixed action played by Player $j$ during $\mathcal B^i_r.$

Paralleling the $m=2$ case, Player $i$ only plays actions in the set $\{m-i+1,m-i+2,\dots,K\}$, and he starts each round with an active arm set consisting only of $\{m-i+1\}$. He performs random explorations $\alpha:=K^{1/2}T^{-1/2m}$ fraction of the time. While playing in the active arm set, he uses an anytime-low-swap-regret algorithm, which achieves expected regret $O(K|\mathcal B^i_1|\sqrt{S\log(K)})$ after $S$ consecutive $i$-blocks. Note that is easy to turn the low-swap-regret algorithm of \cite{Sto05} into an anytime algorithm with the same guarantee by shrinking the learning rate and restarting on a dyadic set of times. Player $i$ also resets his memory every $T^{1/m}$ blocks (or $T^{1-\frac{i-1}{m}}$ timesteps, or every time a new $\mathcal B^{i-1}_{r'}$ block begins).

\subsection{Swap regret of each player}

The lemma we need controls the swap regret of Player $i$ on each $\mathcal B^{i-1}_r$ block.

\begin{lemma} \label{lem:Playeriswapregret}
Player $i$ satisfies, for each $r\leq T^{\frac{i-1}{m}}$,
\[
\max_{\Phi : \{m-i+1, \hdots, K\} \rightarrow \{m-i+1, \hdots, K\}} \E \sum_{t\in \mathcal B^{i-1}_r} \big( \hat{\ell}_t(A_{i,t}) - {\ell}_t(\Phi(A_{i,t})) \big) = O\left(K^{3/2}T^{1-\frac{2i-1}{2m}} \sqrt{\log(K)} \right) \,.
\]

Moreover the expectation takes as given the actions $A_{j,\mathcal B^{i-1}_r}$ for $j\leq i-1$ (which are constant during $\mathcal B^{i-1}_r).$
\end{lemma}

\begin{proof} The proof is similar to that of Lemma~\ref{lem:Bobnocollision}. We first fix a function $\Phi$.

Again set $\mathcal{L}_0(i)$ to be the set of all $t \in \mathcal{B}^{i-1}_r$ such that $\ell_t(i) \not= 1$ and $\mathcal{L}_1(i)$ be the set of all $t \in \mathcal{B}^{i-1}_r$ such that $\ell_t(i) = 1$. Let us define \[ \mathcal{L}_0(i,t) = \left|[t] \cap \mathcal{L}_0(i)\right| \,.\]

Let $t^*(i)$ be the time that $i$ is added to the active set. Note that the active set changes for at most $K$ times during each block, hence partitions $\mathcal B^{i-1}_r$ into $\mathcal B^{i-1}_{r,1}\cup\mathcal B^{i-1}_{r,2}\cup\dots\cup\mathcal B^{i-1}_{r,K}$. So the swap-regret bound $O\left(K\sqrt{|\mathcal B^{i-1}_{r,j}|\log K}\right)$ applies to the set of non-exploration times in $\mathcal B^{i-1}_{r,j}$. Summing over the $K$ subblocks $\mathcal B^{i-1}_{r,i}$ and including the loss from $\mathcal L_0(i,t^*)$ and from exploration rounds themselves, we have  
\begin{eqnarray*}
& \E\left[\sum_{t \in \mathcal{B}^{i-1}_r} \hat{\ell}_t(A_{i,t})  - \sum_{t \in \mathcal{B}^{i-1}_r} \hat{\ell}_t(\Phi(A_{i,t})) \right]\\
 \leq &\sum_i\E[\mathcal{L}_0(i,t^*(i)) ]+O\left(\alpha T^{1-\frac{i-1}{m}}\right)+ O\left(\mathbb E\left[ KT^{1-\frac{i}{m}}\sqrt{\log K}\cdot \sum_{j\leq K} \sqrt{\frac{|\mathcal B^{i-1}_{r,j}|}{|\mathcal B^i_1|}} \right] \right) \,.
\end{eqnarray*}

Recall we set an exploration rate of $\alpha=K^{1/2}T^{-1/2m}$. To estimate the $\E[\mathcal{L}_0(i,t^*(i))]$ terms we observe that each $i$-block contributes at most $T^{1-\frac{i}{m}}$ and is added to the active set with probability at least $\frac{\alpha}{K}=K^{-1/2}T^{-1/2m}$ whenever it gives a positive contribution. Therefore both of the first two terms are $O(K^{3/2}T^{1-\frac{2i-1}{2m}})$ in total.

By Jensen's inequality the last term is maximized when $|\mathcal B^{i-1}_{r,j}|=\frac{T^{1-\frac{i-1}{m}}}{K}$ for all $j$. In this case the final sum has $K$ terms each of size $T^{-\frac{1}{2m}}K^{-1/2}$. Combining, we have

\begin{eqnarray*}
 \E\left[\sum_{t \in \mathcal{B}^{i-1}_r} \hat{\ell}_t(A_{i,t})  - \sum_{t \in \mathcal{B}^{i-1}_r} \hat{\ell}_t(\Phi(A_{i,t})) \right]
 \leq &O\left(K^{3/2}T^{1-\frac{2i-1}{2m}} \sqrt{\log(K)} \right) 
\end{eqnarray*}

This is almost what we need. We also need to control

\[\mathbb E\left[\sum_{t\in\mathcal B^{i-1}_r}\left(\hat\ell_t(\Phi(A_{i,t}))-\ell_t(\Phi(A_{i,t}))\right)\right].\]

This is simply bounded by $|\mathcal B_r^{i-1}|\alpha m$ for the explorations of the $m$ other players - note that the value of $\alpha$ is the same for all players. As $m\leq K$ we have 

\[\mathbb E\left[\sum_{t\in\mathcal B^{i-1}_r}\left(\hat\ell_t(\Phi(A_{i,t}))-\ell_t(\Phi(A_{i,t}))\right)\right]=O\left(K^{3/2}T^{1-\frac{2i-1}{2m}}\right).\]

Adding gives the claimed result.

\end{proof}

\subsection{Assigning a top-$m$ Action to each player}

We now generalize our functions $f,g$ from the two-player case. There, the point was to ensure that $\{f(A_r),g(A_r)\}=\{a,b\}$ for any given actions $\{a,b\}$, and apply this when $a,b$ are the best two actions. This allowed us to compare the regret of Alice against $f$ and the regret of Bob against $g$. Here we describe a more general construction. The construction takes as given a set of $m$ ``optimal" actions $\{a_1,a_2,\dots,a_m\}$ and a sequence of \emph{not necessarily distinct} actions $A_1, \dots, A_m\in [K]$ such that $A_i\geq m-i+1$ for all $i$; the actions $A_i$ represent the actions currently being played by the players. The construction assigns Player $i$ a distinct one of these actions $\tilde A_i=a_j$ for some $j$. Explicitly, for $k=1$ to $k=m$, we:




\begin{enumerate}

	\item Set $\tilde A_k= A_k$ if both $A_k\in \{a_j|j\leq m\}$ and $A_k\notin \{\tilde A_j|j<k\}$. Essentially, we define $\tilde A_k=A_k$ when possible.
	\item In the case that $\tilde A_k= A_k$ does not happen in the previous step, define $\tilde A_k=a_j$ for the smallest value of $j$ such that $a_j\notin \{\tilde A_j|j<k\}$ and $a_j\geq m-k+1$.

\end{enumerate}

\begin{lemma} In the setting above, with actions

\[\{a_1,a_2,\dots,a_m\}\]
\[A_i\geq m-i+1,\text{ for } i\leq m.\]

the following hold:

\begin{enumerate}[label=(\Alph*)]
	\item $\tilde A_k$ has a well-defined value. Furthermore:
	\begin{enumerate}
	    \item $\tilde A_k\in \{a_1,\dots, a_m\}.$
	    \item $\tilde A_k\geq m-k+1.$
	    \item $\tilde A_k\notin \{\tilde A_1,\dots,\tilde A_{k-1}\}.$
	    \item $\tilde A_k\notin \{A_1,\dots,A_{k-1}\}.$
	\end{enumerate}
	\item $\tilde A_k$ is a function of the set $\{a_j|j\leq m\}$ and the sequence $(A_1,\dots, A_k)$.
	\item The set $\{\tilde A_i|i\leq m\}$ is a permutation of the set  $\{a_i|i\leq m\}$.
\end{enumerate}

\end{lemma}

\begin{proof} To see claim $(A)$, we think about choosing $\tilde A_k$. Points $(a),(b),(c)$ are all clear by construction assuming a suitable value of $\tilde A_k$ always exists. So the point is to show a value $\tilde A_k$ exists and is not equal to any $A_j$ for $j\leq k-1$.

To see this, we first observe that if $A_j\in \{a_1,\dots,a_m\}$ for $j\leq k-1$ then $\tilde A_i=A_j$ for some $i\leq k-1$. That is, $A_j$ is never actually available as a value of $\tilde A_k$. Indeed, by construction $\tilde A_j$ would equal $A_j$ if $A_j\in \{a_1,\dots,a_m\}$ and no previous $\tilde A_i$ equaled $A_j$. 

Therefore, any value \[\tilde A_k\in \left(\{a_j|j\leq m\}\cap \{m-k+1,\dots,K\} \right)\backslash \{\tilde A_j|j\leq k-1\}\] satisfying $(a),(b),(c)$ automaticaly satisfies $(d)$. Assuming one exists, step 2 of the algorithm will pick such a value for $\tilde A_k$. So we are left to prove that the above set of possible $\tilde A_k$ values is non-empty. 

We do this with a simple counting argument. Observe that of the three subsets above, the first $\{a_j|j\leq m\}$ has $m$ elements, the second has $K-m+k$, and the third has $k-1$. Intersecting the first two results in a set with at least $k$ elements, and removing $k-1$ leaves at least $1$. We conclude that a value of $\tilde A_k$ making $(a),(b),(c)$ true exists, hence is picked by the algorithm, and this value automatically satisfies $(d)$ as well.


Claim $(B)$ is true by induction, since $\tilde A_k$ is a function of the set $\{a_j|j\leq m\}$, the sequence $(A_1,\dots, A_k)$ and the sequence $(\tilde A_1,\dots,\tilde A_{k-1}).$ 

Claim $(C)$ is implied by claim $(A)$.

\end{proof}

We denote by $\Phi_k$ the functions obtained from the Lemma, which by $(B)$ take as input sets $\{a_j|j\leq m\}$ of actions, and a sequence $(A_1,\dots,A_k)$ of not-necessarily-distinct action. We use the following notation to suggest that $A_k$ is the actual argument, while the rest are (fixed) parameters:

\[\Phi_{k;(a_1,\dots,a_m);(A_1,\dots,A_{k-1})}(A_k)=\Phi_{k;(a_1,\dots,a_m);(A_1,\dots,A_{k})}.\]

Claim $(C)$ of the lemma says that these functions satisfy the property 

\[\{\Phi_{j;(a_1,\dots,a_m);(A_1,\dots,A_{j-1})}(A_j)|j\leq m\}=\{a_j|j\leq m\}\]

as long as $A_j\geq m-j+1$ for all $j$. 

To obtain the desired regret bound for the multiplayer bandit game, we use these functions $\Phi$ as our swap functions, where the actions $A_j$ are those of the slower players. 

\subsection{Proof of Theorem~\ref{thm:manyplayer}}

\label{subsec:manyplayerproof}

\begin{proof} For any distinct actions $(a_1,\dots,a_m)$ we show the regret bound

\[\mathbb E\left[\sum_{t\leq T}\sum_{i\leq m} \left(\hat\ell_t(A_{i,t})-\ell_t(a_i)\right)\right]=O\left(mK^{3/2}T^{1-\frac{1}{2m}}\sqrt{\log(K)}\right).\]

For player $i$, consider each block $\mathcal B^{i-1}_r$, and apply Lemma~\ref{lem:Playeriswapregret} with the $\Phi$ function above,

\[\Phi_i(A_i)=\Phi_{i;(a_1,\dots,a_m);(A_{1,\mathcal B^{i-1}_r},\dots,A_{i-1,\mathcal B^{i-1}_r})}(A_i).\]

Note that in constructing $\Phi$ we allowed the sequence $A_1,\dots,A_{i-1}$ to have repeats, which might happen here if some Player $j$ for $j\leq i-1$ is exploring outside his active arm set.

We let $\Phi_{i,t}$ be the function $\Phi_i$ during for the block containing time $t$. For each fixed $i$, summing over all $T^{\frac{i-1}{m}}$ blocks $\mathcal B^{i-1}_r$ for varying $r$ shows

\begin{eqnarray*}
 \E\left[\sum_{t \leq T} \hat{\ell}_t(A_{i,t})  - \sum_{t \leq T} {\ell}_t(\Phi_{i,t}(A_{i,t})) \right]
 \leq &O\left(K^{3/2}T^{1-\frac{1}{2m}} \sqrt{\log(K)} \right) 
\end{eqnarray*}

 By construction, at each time $t$ the functions $\{\Phi_{j,t}|j\leq m\}$ take all the values $\{a_j|j\leq m\}$ exactly once. Therefore summing the previous inequality over $i\leq m$ gives the claimed regret bound.

\end{proof}


\section*{Acknowledgement}
This work was supported in part by an NSF graduate fellowship and a Stanford graduate fellowship.

\bibliographystyle{plainnat}
\bibliography{newbib}

\begin{thebibliography}{23}
\providecommand{\natexlab}[1]{#1}
\providecommand{\url}[1]{\texttt{#1}}
\expandafter\ifx\csname urlstyle\endcsname\relax
  \providecommand{\doi}[1]{doi: #1}\else
  \providecommand{\doi}{doi: \begingroup \urlstyle{rm}\Url}\fi

\bibitem[Alatur et~al.(2019)Alatur, Levy, and Krause]{ALK19}
P.~Alatur, K.~Y. Levy, and A.~Krause.
\newblock Multi-player bandits: The adversarial case.
\newblock \emph{arXiv preprint arXiv:1902.08036}, 2019.

\bibitem[{Anandkumar} et~al.(2011){Anandkumar}, {Michael}, {Tang}, and
  {Swami}]{AMTS11}
A.~{Anandkumar}, N.~{Michael}, A.~K. {Tang}, and A.~{Swami}.
\newblock Distributed algorithms for learning and cognitive medium access with
  logarithmic regret.
\newblock \emph{IEEE Journal on Selected Areas in Communications}, 29\penalty0
  (4):\penalty0 731--745, 2011.

\bibitem[{Anantharam} et~al.(1987){Anantharam}, {Varaiya}, and
  {Walrand}]{AVW87}
V.~{Anantharam}, P.~{Varaiya}, and J.~{Walrand}.
\newblock Asymptotically efficient allocation rules for the multiarmed bandit
  problem with multiple plays-part i: I.i.d. rewards.
\newblock \emph{IEEE Transactions on Automatic Control}, 32\penalty0
  (11):\penalty0 968--976, 1987.

\bibitem[Audibert et~al.(2014)Audibert, Bubeck, and Lugosi]{ABL14}
J.Y. Audibert, S.~Bubeck, and G.~Lugosi.
\newblock Regret in online combinatorial optimization.
\newblock \emph{Mathematics of Operations Research}, 39:\penalty0 31--45, 2014.

\bibitem[Auer et~al.(2002)Auer, Cesa-Bianchi, Freund, and Schapire]{ACFS03}
P.~Auer, N.~Cesa-Bianchi, Y.~Freund, and R.~Schapire.
\newblock The non-stochastic multi-armed bandit problem.
\newblock \emph{SIAM Journal on Computing}, 32\penalty0 (1):\penalty0 48--77,
  2002.

\bibitem[Avner and Mannor(2014)]{AM14}
O.~Avner and S.~Mannor.
\newblock Concurrent bandits and cognitive radio networks.
\newblock In \emph{ECML/PKDD}, 2014.

\bibitem[Blum and Mansour(2007)]{BM07}
A.~Blum and Y.~Mansour.
\newblock From external to internal regret.
\newblock \emph{Journal of Machine Learning Research (JMLR)}, 8:\penalty0
  1307--1324, 2007.

\bibitem[Bonnefoi et~al.(2017)Bonnefoi, Besson, Moy, Kaufmann, and
  Palicot]{BBMKP17}
R.~Bonnefoi, L.~Besson, C.~Moy, E.~Kaufmann, and J.~Palicot.
\newblock Multi-armed bandit learning in iot networks: Learning helps even in
  non-stationary settings.
\newblock In \emph{International Conference on Cognitive Radio Oriented
  Wireless Networks}, pages 173--185. Springer, 2017.

\bibitem[Boursier and Perchet(2018)]{BP18}
E.~Boursier and V.~Perchet.
\newblock Sic-mmab: Synchronisation involves communication in multiplayer
  multi-armed bandits.
\newblock \emph{arXiv preprint arXiv:1809.08151}, 2018.

\bibitem[Bubeck and Cesa-Bianchi(2012)]{BC12}
S.~Bubeck and N.~Cesa-Bianchi.
\newblock Regret analysis of stochastic and nonstochastic multi-armed bandit
  problems.
\newblock \emph{Foundations and Trends in Machine Learning}, 5\penalty0
  (1):\penalty0 1--122, 2012.

\bibitem[Cesa-Bianchi and Lugosi(2012)]{CL11}
N.~Cesa-Bianchi and G.~Lugosi.
\newblock Combinatorial bandits.
\newblock \emph{Journal of Computer and System Sciences}, 78\penalty0
  (5):\penalty0 1404--1422, 2012.

\bibitem[Dekel et~al.(2014)Dekel, Ding, Koren, and Peres]{DDKP14}
O.~Dekel, J.~Ding, T.~Koren, and Y.~Peres.
\newblock Bandits with switching costs: $t^{2/3}$ regret.
\newblock In \emph{Proceedings of the Forty-sixth Annual ACM Symposium on
  Theory of Computing}, STOC '14, pages 459--467. ACM, 2014.

\bibitem[Geulen et~al.(2010)Geulen, V{\"o}cking, and Winkler]{GVW10}
S.~Geulen, B.~V{\"o}cking, and M.~Winkler.
\newblock Regret minimization for online buffering problems using the weighted
  majority algorithm.
\newblock In \emph{COLT}, pages 132--143, 2010.

\bibitem[Goldreich(2010)]{Gol10}
Oded Goldreich.
\newblock \emph{A primer on pseudorandom generators}, volume~55.
\newblock American Mathematical Soc., 2010.

\bibitem[Kleinberg et~al.(2010)Kleinberg, Niculescu-Mizil, and Sharma]{KNS10}
R.~Kleinberg, A.~Niculescu-Mizil, and Y.~Sharma.
\newblock Regret bounds for sleeping experts and bandits.
\newblock \emph{Machine Learning}, 80:\penalty0 245--272, 2010.

\bibitem[{Lai} et~al.(2008){Lai}, {Jiang}, and {Poor}]{LJP08}
L.~{Lai}, H.~{Jiang}, and H.~V. {Poor}.
\newblock Medium access in cognitive radio networks: A competitive multi-armed
  bandit framework.
\newblock In \emph{2008 42nd Asilomar Conference on Signals, Systems and
  Computers}, pages 98--102, 2008.

\bibitem[Lattimore and Szepesv{\'a}ri(2019)]{LS19}
T.~Lattimore and Cs. Szepesv{\'a}ri.
\newblock \emph{Bandit Algorithms}.
\newblock Cambridge University Press (preprint), 2019.

\bibitem[{Liu} and {Zhao}(2010)]{LZ10}
K.~{Liu} and Q.~{Zhao}.
\newblock Distributed learning in multi-armed bandit with multiple players.
\newblock \emph{IEEE Transactions on Signal Processing}, 58\penalty0
  (11):\penalty0 5667--5681, 2010.

\bibitem[Lugosi and Mehrabian(2018)]{LM18}
G.~Lugosi and A.~Mehrabian.
\newblock Multiplayer bandits without observing collision information.
\newblock \emph{arXiv preprint arXiv:1808.08416}, 2018.

\bibitem[Robbins(1952)]{Rob52}
H.~Robbins.
\newblock Some aspects of the sequential design of experiments.
\newblock \emph{Bulletin of the American Mathematics Society}, 58:\penalty0
  527--535, 1952.

\bibitem[Rosenski et~al.(2016)Rosenski, Shamir, and Szlak]{RSS16}
J.~Rosenski, O.~Shamir, and L.~Szlak.
\newblock Multi-player bandits - a musical chairs approach.
\newblock In \emph{ICML}, 2016.

\bibitem[Stoltz(2005)]{Sto05}
G.~Stoltz.
\newblock \emph{Incomplete Information and Internal Regret in Prediction of
  Individual Sequences}.
\newblock PhD thesis, Universit\'e Paris-Sud, 2005.

\bibitem[Uchiya et~al.(2010)Uchiya, Nakamura, and Kudo]{UNK10}
T.~Uchiya, A.~Nakamura, and M.~Kudo.
\newblock Algorithms for adversarial bandit problems with multiple plays.
\newblock In \emph{Proceedings of the 21st International Conference on
  Algorithmic Learning Theory (ALT)}, 2010.

\end{thebibliography}
\end{document}